\documentclass[article,submission,letterpaper]{imsart}

\usepackage[utf8]{inputenc}
\RequirePackage[OT1]{fontenc}
\RequirePackage{amsthm,amsmath}
\RequirePackage[numbers]{natbib}
\RequirePackage[colorlinks,citecolor=blue,urlcolor=blue]{hyperref}
\usepackage{amsmath}
\usepackage{amsfonts}
\usepackage{amssymb}
\usepackage{epsf}
\usepackage[left]{eurosym}
\usepackage{multirow}
\usepackage{graphicx}
\usepackage{csvsimple}
\usepackage{algorithm}
\usepackage{algpseudocode}
\usepackage{float}
\usepackage{cancel}
\usepackage{grffile}
\usepackage{amsmath}
\usepackage{amsfonts}
\usepackage{amssymb}
\usepackage{epsf}
\usepackage[left]{eurosym}
\usepackage{mathtools}
\usepackage{booktabs}
\usepackage[normalem]{ulem}
\usepackage{bbm}

\include{macrogilles_main}

\newtheorem{mydef}[theorem]{Definition}
\newtheorem{cor}[theorem]{Corollary}
\newcommand{\pfun}{\pi}

\newcommand{\HS}{\mathrm{HS}}
\newcommand{\BV}{\mathrm{BV}}
\newcommand{\Lip}{\mathrm{Lip}}
\newcommand{\diag}{\mathrm{diag}}

\startlocaldefs
\numberwithin{equation}{section}
\newtheorem{thm}[theorem]{Theorem}
\endlocaldefs

\begin{document}

\begin{frontmatter}
  \title{Concentration of weakly dependent Banach-valued sums and applications to
  statistical learning methods}
\runtitle{Concentration of weakly dependent Banach-valued sums }
\begin{aug}
  \author{\fnms{Gilles}  \snm{Blanchard}\thanksref{t1} 
    \ead[label=e1]{blanchard@uni-potsdam.de}},
  \author{\fnms{Oleksandr} \snm{Zadorozhnyi}\thanksref{t2} 
    \ead[label=e2]{zadorozh@uni-potsdam.de}}
	
	\runauthor{Blanchard, Zadorozhnyi}
	
	\affiliation{University of Potsdam}
	
	\address{University of Potsdam, Institute of Mathematics\\
          Karl-Liebknecht-Str. 24/25, Potsdam, 14476. \printead{e1,e2}}
	

\thankstext{t1}{Acknowledges partial support from the DFG FOR-1735 research group ``Structural Inference in Statistics''. This work was finished while GB was a guest at the Institut des Hautes Études Scientifiques, Université Paris-Saclay.}
\thankstext{t2}{Acknowledges full support from the  DFG CRC-1294 collaborative research center ``Data Assimilation''.}	
\end{aug}

%
%
%
%
%
%

\begin{abstract}
  We obtain a  Bernstein-type inequality for sums of Banach-valued random variables
  satisfying a weak dependence assumption of general type and
  under certain smoothness assumptions of the underlying Banach norm.
  We use this inequality in order to investigate in the asymptotical regime the error upper bounds  for the broad family of spectral regularization methods for
  reproducing kernel decision rules,
  when trained on a sample coming from a 
  $\tau-$mixing process.
\end{abstract}
\begin{keyword}[class=MSC]
\kwd[primary ]{60E15}
\kwd[; secondary ]{62M10}
\kwd{62G08}
\kwd{60G10}
\kwd{68T05}
\end{keyword}
%
\begin{keyword}
\kwd{weak dependence}
\kwd{concentration}
\kwd{spectral regularization}
\kwd{Bernstein inequality}
\kwd{Banach-valued process}
\end{keyword}

\end{frontmatter}

\section{Introduction}


Let $(X_{k})_{k \in \mathbb{N}_{+}}$ be an 
 integrable and centered stochastic process taking
values in a separable Banach space $\paren{\mathcal{B},\norm{\cdot}}$.
Define $S_{n} = X_{1} + X_{2}+\ldots + X_{n}$. In this work, we are interested in the non-asymptotic behaviour of the deviations of $S_{n}$ from zero in $\mathcal{B}$; more precisely, we investigate exponential concentration inequalities for events of the type $\set{\norm{S_{n}} \geq t}$, for $t>0$.
In the simplest situation where $(X_{1},X_{2},\ldots,X_{n})$ are mutually independent and
real-valued,
the celebrated  Hoeffding's \cite{Hoeffding:63} and
Bernstein's inequalities \cite{Bernstein:24}  are available.
Vector-valued analogues (in finite or infinite dimension) of those concentration inequalities for norms of sums of independent random variables were first established for the case of bounded independent random variables in Hilbert spaces by Yurinskyi \cite{Yurinskyi:70}.
 
The situation differs in an arbitrary Banach space. There, the distribution of $\norm{S_n}$ (in particular its expectation) heavily depends on the geometry of the underlying Banach space, and moment (Bernstein-like) conditions for the individual variables $X_i$ are generally not sufficient for a generic control of $\norm{S_{n}}$ around zero (see \cite{Yurinskyi:95}, Example 3.0.1). Still, under assumptions on  the "smoothness" of the underlying Banach norm (reflected by boundedness of its first two Gâteaux-derivatives), one can control the deviations of $\norm{S_n}$ around zero.
Corresponding  concentration inequalities have been obtained in
\cite{Pinelis:86} and \cite{Pinelis:92}. 

Of interest for many applications is the case where random samples are generated from some non-trivial stochastic process with (possibly infinite) memory.  The generalization of Hoeffding's inequality for real-valued martingales and martingale differences together with its application to least squares estimators in linear and smooth autoregressive models are presented in  \cite{vandeGeer02}. An extension of the Hoeffding-Azuma inequalities for the weighted sum of uniformly bounded martingale differences can be found in \cite{Rio:13}. Generalizations of the exponential inequalities for the case of real-valued supermartingales were obtained in \cite{Fredman:75} and recently generalized in \cite{Fan:15}, where the authors use change of probability measure techniques, and give applications for estimation in the general parametric (real-valued) autoregressive model. Extensions of \cite{Fredman:75} for the case of supermartingales in Banach spaces were obtained in \cite{Pinelis:94}.

Beyond the (super)martingale setting, the need to handle more general processes which have some "asymptotic independence" assumptions led to the concept of mixing. Definitions of (strong) $\alpha-$, $\phi-$ and $\rho-$ mixing were introduced in \cite{Rosenblatt:56}, \cite{Ibragimov:59},\cite{Kolmogorov:60}, we refer also to \cite{Bradley:05} for a broad survey about the properties and relations between 
 strong mixing processes. However, there are examples of dynamical systems \cite{Dedecker:07} generated by uniformly expanding maps that are not even $\alpha-$mixing (considered the weakest form of strong mixing assumptions). Such type of processes include mixingales \cite{Andrews:88,McLeish:75}, associated processes \cite{Esary:85,Fortuyn:71}, and various more recent notions of weak dependence \cite{Rio:96, Bickel:99,Doukhan:99}.
In this paper, we consider the analysis of the inherent dependency of the random sample by means of a general type of weakly dependent process. In this general framework, 
many techniques used in the independent data scenario were improved and combined with other methods to obtain concentration inequalities for the sum of \textit{real-valued} random variables. For example, generalizations of Bernstein's inequality for $\phi$-mixing random
processes were obtained 
combining
 the entropy method 
  with the blocking technique in \cite{Samson:00}; using a similar blocking technique ensuring asymptotic independence, Bernstein-type inequalities for geometrically $\alpha-$mixing processes and moderate deviation principles were derived in \cite{Merlevede:09}; deviation inequalities for real-valued sums of variables from general $\alpha-$mixing processes were obtained in \cite{Bosq:93} through approximation by independent random sums and the blocking technique.
Moreover, in \cite{Steinwart:16} the blocking technique together with majorization of joint distributions by means of the marginals and a general Chernoff's bounding principle are used to obtain Bernstein-type inequalities for real-valued Lipschitz functions
of $\mathcal{C}-$mixing processes. In \cite{Kontorovich:08}, the martingale difference method is used to establish general McDiarmid-type concentration inequalities for real-valued Lipschitz functions of dependent random sequences on a countable state space. Using logarithmic Sobolev inequalities and the contractivity condition related to Dobrushin and Shlosman's strong mixing assumptions, general non-product measure concentration inequalities were obtained in \cite{Marton:04}. 

Most of the 
above mentioned inequalities characterize the deviations of sums of real-valued random variables.
Concerning Hilbert- or Banach-valued weakly dependent processes, a significant literature exists on limit theorems of central limit or Berry-Esseen type, motivated in particular by functional time series
\cite{Bosq2000, Horvath2012}; we will limit ourselves to pointing out the recent reference \cite{Jirak2018} and
the substantial literature review there. In this paper, we are specifically interested in concentration
results with a non-asymptotic control of exponentially decaying deviation probabilities.
A few results concern the concentration of real-valued functional of weakly dependent variables over somewhat general spaces, and can be applied to norms of sums of vector-valued variables. This is the case for the measure concentration result in \cite{Kontorovich:08} for
so-called $\eta-$mixing (which is implied by $\phi-$mixing) random variables, but a condition
called $\Psi$-dominance \cite{Kontorovich:06} must hold (it is satisfied if the underlying variable space is countable,
or is a closed subspace of the real line). This result implies Hoeffding-Azuma type inequalities for norms of
sums. Still, to the best of our knowledge, it is unknown how these mixing assumptions are connected to $\alpha-$, $\beta-$ or $\Phi_{\mathcal{C}}-$mixing, or whether they can be applied to norms in arbitrary Banach spaces. The aforementioned measure concentration results of \cite{Marton:04} for distributions of dependent real variables with continuous density imply concentration of the norm of their sum (which is a Lipschitz function in Euclidean distance) in an Euclidean space.
However, the question becomes more challenging when one considers concentration of the norm of random variables in a separable, infinite-dimensional space. Finally, the recent work \cite{Dedecker:15} establishes a Hoeffding-type bound under assumptions close to what we consider here; we underline that we are interested in sharper Bernstein-type  rather than Hoeffding-type bounds (see also Section~\ref{se:disc} for a more detailed discussion of the latter work).

This paper is organized as follows: in Section~\ref{sec:preliminaries_and_notations}, we
recall the setting for stochastic processes with values in a Banach space. We recall the
definition from \cite{Deschamps:06} to consider a general type of weakly dependent processes. In Section~\ref{sec:main_results}, we pose the main assumptions about the structure of the underlying infinite dimensional Banach space and present in a general form the new Bernstein-type inequalities for $\mathcal{C}-$mixing processes. Furthermore, here we also provide specific corollaries for the cases of either exponentially (geometrically) or polynomially mixing decay rates. We compare our results to the former inequalities on the concentration of real-valued $\mathcal{C}$-mixing processes. As an application,
in Section~\ref{sec:application} we investigate the (asymptotical) error bounds for reproducing kernel learning algorithms using a general form of spectral regularization when the sample is drawn from a process which satisfies the so-called $\tau$-mixing assumption. All proofs can be found in the Appendix. 

\section{Preliminaries and Notations}
\label{sec:preliminaries_and_notations}
Let $\left(\Omega,\mathcal{F},\mathbb{P}\right)$ be a probability space. We recall that $\paren{\mathcal{B},\norm{\cdot}}$ is some separable Banach space 
and $\mathcal{X}\subset \cB$ a ball of $\cB$. 
We use the standard notions of $p$-integrable and essentially bounded real functions spaces and use the notation $L_{p}(\mathbb{P})  :=L_{p}(\Omega, \mathcal{F},\mathbb{P})$ and $L_{\infty} := L_{\infty}(\Omega, \mathcal{F},\mathbb{P})$.
%
  Following \cite{Deschamps:06}, we define mixing processes with respect to a class  a of real-valued functions. Let $C(\cdot)$ be a semi-norm over a closed subspace $\mathcal{C}$ of the Banach space of bounded real-valued functions $f: \mathcal{X} \mapsto \mathbb{R}$. We define the $\mathcal{C}$-norm by $\norm{f}_{\mathcal{C}} := \norm{f} + C(f)$, where $\norm{\cdot}$ is the supremum norm on $\mathcal{C}$,
  and introduce $\mathcal{C}_{1} = \{ f \in \mathcal{C}, C(f) \leq 1\}$. 

Define $\cM_{j}= \sigma(X_{i}: 1 \leq i \leq j), j \in \mathbb{N}$ to be the sigma-algebra generated by the random variables $X_{1},\ldots,X_{j}$. 

\begin{mydef}
	\label{def:gen+phi_mixing}
	For $k \in \mathbb{N}_{>0}$ we define the ${\mathcal{C}}$-mixing coefficients as
\begin{multline*}
  \Phi_{\mathcal{C}}(k) = \sup \big\{ \ee{}{Y\varphi(X_{i+k})} - \ee{}{Y}\ee{}{\varphi(X_{i+k})} \mid \\
  i\geq 1, Y \in L_{1}(\Omega,\mathcal{M}_{i},\mathbb{P}), \norm{Y}_{L_{1}(\mathbb{P})} \leq 1, \varphi \in \mathcal{C}_{1} \big\}.
\end{multline*}
\end{mydef}
We say that the process $(X_{i})_{i \geq 1}$ is $\Phi_{\mathcal{C}}-$mixing (or simply $\mathcal{C}$-mixing) if $\lim_{k \rightarrow \infty} \Phi_{\mathcal{C}}(k) = 0$.
If $\Phi_{\mathcal{C}}(k) \leq c\exp(-bk^{\gamma})$ for some constants $b,\gamma > 0$, $c\geq 0$ and  all $k \in \mathbb{N}$,  then a stochastic process $(X_{k})_{k \geq 1}$ is said to be \textit{exponentially} (or \textit{geometrically}) ${\mathcal{C}}$-mixing. If $\Phi_{\mathcal{C}}(k) \leq ck^{-\gamma}$ for all $k \in \mathbb{N}$ and for some constants $c\geq 0, \gamma >0$, then the stochastic process $(X_{k})_{k\geq 1}$ is said to be \textit{polynomially} ${\mathcal{C}}$-mixing.

As discussed in \cite{Deschamps:06}, ${\mathcal{C}}$-mixing describes many natural time-evolving systems and finds its application for a variety of dynamical systems. The authors of \cite{Steinwart:16} use a slighty different definition of $\Phi_{\mathcal{C}}$-mixing coefficient, where the supremum is taken over the class of functions $\{f: \norm{f}_{\mathcal{C}} \leq 1 \}$.

Thus, dependency coefficients $\Phi_{\mathcal{C}}$ are characterized by the control over correlations between the past and one moment in the future of the process, for functions of bounded supremum norm from class $\mathcal{C}_{1}$. A fundamental result (\cite{Deschamps:06}, Lemma~1.1.2) claims that 
Definition~\ref{def:gen+phi_mixing} can be equivalently stated as  following:
\begin{mydef}[Equivalent to  Definition~\ref{def:gen+phi_mixing}]
	\label{def:gen_phi_mixing_alter}
	\begin{equation*}
	  \Phi_{\mathcal{C}}(k) = \sup \big\{ \norm{E[\varphi (X_{i+k})|\cM_{i}]-E[\varphi (X_{i+k})]}_{\infty}
          \mid \varphi \in \mathcal{C}_{1}, i\geq 1 \big\},
	\end{equation*}
\end{mydef}
where $\norm{\cdot}_{\infty}$ is the $L_{\infty}(\mathbb{P})$ norm.
In our theoretical analysis we will use Definition~\ref{def:gen_phi_mixing_alter} for processes which are assumed to be $\mathcal{C}-$mixing. We first describe some examples of semi-norms  ${C}$.

\begin{example}
  \label{example2.1}
  Let $\mathcal{C}_{\Lip}$ be the set of bounded Lipschitz functions over $\mathcal{X}$. Consider
  \begin{align*}
  	C_{\Lip}(f) := \norm{f}_{\Lip(\mathcal{X})} = \sup \bigg\{\frac{|f(s)-f(t)|}{\norm{s-t}} \biggm| s,t \in \mathcal{X}, s \neq t \bigg\}.
  \end{align*}
  It is easy to see that $C_{\Lip}(f)$ is a semi-norm.
  With this choice of class $\mathcal{C}$ and semi-norm $C(\cdot)$, we obtain the so-called
 $\tau-$mixing coefficients (see \cite{Dedecker:07} and  \cite{Wintenberger:10} for the real-valued case), which will be denoted $\tau(k):=\Phi_{\cC}(k), k\geq 1$.
\end{example}
 
  \paragraph{Examples of $\tau-$mixing sequences.}
   Consider a Banach-valued auto-regressive process of order $1$: 
  \begin{equation*}
  X_{i} = \rho (X_{i-1}) + \xi_{i}, \text{for } i \in \mathbb{Z},
  \end{equation*}
  where $(\xi_{i})_{i \in \mathbb{Z}}$ is an i.i.d. sequence such that $\norm{\xi}_{} \leq 1$ almost surely, and~$\rho:\mathcal{X} \mapsto \mathcal{X}$ is a linear operator with $\norm{\rho}_{\star} < 1$, where $\norm{\cdot}_{\star}$ is the operator norm. 
  Due to the linearity of $\rho$,  we can write $X_{t+s}=X_{t,s} + \rho^s (X_{t})$, where $X_{t,s} = \sum_{l=0}^{s-1}\rho^l(\xi_{t+s-l})$.
  For the $\tau$-mixing coefficients, by using this decomposition and the  independence
  $X_{t,s}$ and $X_{t}$, we get: 
  \begin{align*}
  \begin{aligned}
  \label{eq: tau_mixing_ar1}
  \tau(s) &=\sup_{f \in \mathcal{C}_1}\{\norm{E[f(X_{t+s})|\mathcal{M}_{t}] - E[f(X_{t+s})]}_{\infty}\} \\
  & = \sup_{f \in \mathcal{C}_1}\{\norm{E[f(X_{t,s} + \rho^s(X_{t}))|\mathcal{M}_{t}] - E[f(X_{t,s} + \rho^s(X_{t}))]}_{\infty}\} \\
  & = \sup_{f \in \mathcal{C}_1}\{\|E[f(X_{t,s} + \rho^s(X_{t}))- f(X_{t,s}) |\mathcal{M}_{t}] \\
    & \qquad \qquad - E[f(X_{t,s}+\rho^s (X_{t})) - f(X_{t,s})]\|_{\infty}\} \\
  & \leq 2 \norm{\rho^s ( X_{t})}_{\infty} \leq \norm{\rho}_{\star}^{s} \norm{X_{t}}_{\infty} \rightarrow 0,
  \end{aligned}
  \end{align*}
  when $s \rightarrow \infty$, as $X_{t}$ is almost surely bounded. From this we observe that  $(X_{t})_{t \geq 1}$ is exponentially $\tau-$mixing Banach-valued process.  Repeating arguments from \cite{Andrews:84}
  (in the real-valued case), one can show that this process is not always $\alpha-$mixing
  (in particular when $\xi_i$ has a discrete distribution).
  Similarly to the aforementioned argument, it is easy to check that a Hilbert-valued version of the moving-average process of finite order $q< \infty$: 
  \begin{align*}
	  W_{i} = \mu + \sum_{j=0}^{q}\theta_{i-j}\psi_{i-j}, \text{for } i \in \mathbb{Z},
  \end{align*}
  where $\paren{\psi_{j}}_{j \in \mathbb{Z}}$ is an independent and centered noise process
   and $\mu$ is some fixed element in a Hilbert space, is an exponentially $\tau-$mixing process. Furthermore, one can straightforwardly check that $\paren{W_{i}}_{i \in \mathbb{Z}}$ is not a martingale in general.
  
\textit{Remark.}  We observe that the $\tau-$mixing property of the process $\paren{X_{t}}_{t \geq 0}$ is preserved under a $1$-Lipschitz map. More precisely, let $\phi: \mathcal{X} \mapsto \mathcal{H}$ be a $1-$Lipschitz mapping  of the original process  $\paren{X_{t}}_{t \geq 0}$ to some Polish space $\paren{\mathcal{H},\norm{\cdot}_{\mathcal{H}}}$. Then, it is straightforward to check that the process $\paren{\phi\paren{X_{t}}}_{t\geq 0}$ is again $\tau-$mixing. This conservation property is due to the definition of $\tau$-mixing. The concentration inequality of Theorem \ref{thm:main_result_gen} will allow us in Section~\ref{sec:application} to obtain qualitative results about the statistical properties (error bounds) of the estimators of regression function in 
a reproducing kernel Hilbert space. The key idea here is that the estimators of the target function are based on a non-linear but Lipschitz mapping of the corresponding training data sequence into the Hilbert space.

\begin{example}
  \label{example2.2}
Assume $\mathcal{X} \subset \mathbb{R}$ to be an interval on the real line, let $\mathcal{C}_{\BV}:=\BV(\mathcal{X})$ be the set of functions over $\mathcal{X}$ whose total variation is bounded and $C_{\BV}(\cdot)$ be the total variation seminorm:
\begin{align*}
C_{\BV}(f) := \norm{f}_{\mathrm{TV}} = \sup_{(x_{0},\ldots,x_{n}) \in \triangle } \sum_{i=1}^{n}\abs{f(x_{i}) - f(x_{i-1})},
\end{align*}
where $\triangle = \{(x_{0},x_{1},\ldots,x_{n}) \in \mathcal{X}^{n}\mid x_{0} < x_{1} < \ldots < x_{n}\}$. It is known that $\BV(\mathcal{X})$ endowed with the norm $\norm{f}_{\BV} = \norm{f} + C_{\BV}(f)$ is a Banach space. With this choice of $(\mathcal{C},C(\cdot))$ we obtain the
so-called $\tilde{\phi}$-mixing processes, described in \cite{Rio:96}. 
\end{example}




\section{Main assumptions and results}

\label{sec:main_results}

\subsection{Assumptions}

Following \cite{Pinelis:92}, we introduce suitable hypotheses pertaining to the geometry of the underlying Banach space $(\mathcal{B}, \norm{\cdot})$, the distribution of the norm of coordinates $\norm{X_{i}}$, and additional conditions on the considered $C(\cdot)$-semi-norm.

We recall briefly the concept of G\^{a}teaux derivative: for a real-valued function
$f : \mathcal{X} \rightarrow \mathbb{R}$ we say that $f$ is \textit{G\^{a}teaux differentiable} at point $x \in int{(\mathcal{X})}$ in the direction $v \in \mathcal{B}$, if $t\mapsto f(x+tv)$ is
differentiable in $0$. We then denote
$$ 
\delta_{v}f(x) =\left. \frac{d}{dt}\right\vert_{t=0}f(x+tv).
$$
We say that the function $f$ is\textit{ G\^{a}teaux-differentiable} at point $x$ if all the directional derivatives exist and form a bounded linear functional, i.e. an element $D_xf$ in
the dual $\cB^*$ such that $\forall v \in \mathcal{B}$:
$$ 
\lim_{t\rightarrow0}\frac{f(x+tv) - f(x)}{t} = \inner{D_{x}f,v}.
$$ 
In this case $D_{x}f$ is called \textit{G\^{a}teaux derivative} of function $f$ at point $x$. 

{\bf Assumption A1.}
The norm $\norm{\cdot}$ in the Banach space $\mathcal{B}$ is twice Gâteaux differentiable
at every nonzero point in all directions and there exist constants $A_1\geq 1,A_2>0$
such that the following conditions are fulfilled for all $x,v \in \mathcal{B},x \neq 0 $:
\begin{align*}
& |\delta_{v}(\norm{x})| \leq A_{1}\norm{v}, \text{ or equivalently }
\norm[1]{\paren{D_{x}\norm{\cdot}}}_{\star} \leq A_{1};  \\
& |\delta_{v,v}(\norm{x})| \leq A_{2}\frac{\norm{v}^{2}}{\norm{x}},
\end{align*} 
where $\delta_{v,v}$ denotes the second G\^{a}teaux differential in the direction $v$ and $\norm{\cdot}_{\star}$ is
the norm in the dual space $\mathcal{B}^\star$.

We recall  the following examples of Banach spaces that fulfill the desired properties  (see \cite{Pinelis:92}): 

\begin{example}
  \label{example3.1}
Let $\mathcal{B} = \mathbb{H}$ be a separable infinite dimensional Hilbert space with scalar product $\langle \cdot, \cdot \rangle_{\mathbb{H}}$ and  norm $\norm{\cdot}_{\mathbb{H}}$. Then by the Cauchy-Schwartz inequality, it holds: 
\[
\delta_{g}(\norm{f}_{\mathbb{H}}) = \left.\frac{d}{dt}\paren{\sqrt{\langle f+tg,f+tg\rangle}}\right\vert_{t=0}\leq \norm{g}_{\mathbb{H}},
\]
and also 
\[
\delta_{g,g}(\norm{f}_{\mathbb{H}}) = \left. \frac{d}{dt} \paren{\frac{\inner{ f,g} + t\norm{g}_{\mathbb{H}}^2}{\norm{f+tg}_{\mathbb{H}}} }\right\vert_{t=0}\leq \frac{\norm{g}_{\mathbb{H}}^2}{\norm{f}_{\mathbb{H}}} ,
\]
hence $\mathbb{H}$ satisfies Assumption {\bf A1} with constants $A_{1} = A_{2} =1$.
\end{example}

\begin{example}
  \label{example3.2}
Let $\mathcal{B} = L_{p}(\Omega,\mathcal{F},\mathbb{P}), p \geq 2$. Then for any
$ f,g \in \mathcal{B}$ such that $f \neq 0$, it holds: 
\begin{align*}
\begin{aligned}
\delta_{g}(\norm{f}_{p}) &= \frac{d}{dt}\paren[3]{\paren[3]{\int|f + tg|^{p}d\mathbb{P}}^{\frac{1}{p}}}\bigg|_{t=0} = \norm{f}_{p}^{1-p} \int |f|^{p-2} fg d\mathbb{P} \\
& \leq \norm{f}_{p}^{1-p} \norm{f}_{p}^{p-1}\norm{g}_{p} = \norm{g}_{p},
\end{aligned}
\end{align*}
because of H\"older's inequality; similarly: 
\begin{align*}
\begin{aligned}
\delta_{g,g}(\norm{f}_{p}) &= (p-1)\norm{f}_{p}^{1-2p} \paren{\norm{f}_{p}\int |f|^{p-2}g^2 d\mathbb{P} - \paren{\int |f|^{p-2}fg d\mathbb{P}}^2 } \\
& \leq (p-1)\norm{f}_{p}^{1-2p} \paren{\norm{f}_{p}\int |f|^{p-2}g^2 d\mathbb{P}} \leq (p-1)\norm{g}^{2}_{p}\norm{f}_{p}^{-1}.
\end{aligned}
\end{align*}
Thus for $p\geq 2$ an $L_{p}$-space satisfies conditions of Assumption {\bf A1} with constants $A_{1} = 1$, $A_{2} = p-1$.
\end{example}
The next example belonging to random matrix theory was not present in \cite{Pinelis:92} and is apparently new.
It can be given as a relevant application in its own right of the results of
\cite{Pinelis:92} (in the independent case) as well as of the present
results.
\begin{example}
	\label{ex:shatten}
	Let $p\geq 2$ be fixed and $\mathcal{B}$ be the space of real symmetric matrices of dimension $d$ equipped with  the Schatten $p-$norm $\norm{X}_{{p}} =  \paren{\tr\paren{\abs{X}^{p}}}^{\frac{1}{p}}= \paren[1]{\sum_{i=1}^{d}\abs{\lambda_{i}\paren{X}}^{p}}^{\frac{1}{p}}$.
        Then it holds that for any elements $X,H \in \mathcal{B}$, $X\neq \mathbf{0}$: 
	\begin{align*}
		\delta_{H}\paren[1]{\norm{X}_{{p}}} &\leq \norm{H}_{{p}}, \\
		\delta_{H,H}\paren[1]{\norm{X}_{{p}}} & \leq  3\paren{p-1}\frac{\norm{H}_{{p}}^{2}}{\norm{X}_{{p}}},
	\end{align*}
	so  the conditions of Assumption \textbf{A1} are satisfied with constants $A_{1}=1$ and $A_{2} = 3(p-1)$ (for a detailed justification, see Appendix~\ref{app:schatten}).
\end{example}
The conditions in Assumption {\bf A2} are common in the framework of Bernstein-type inequalities.

\textbf{Assumption A2.} There exist positive real constants $c,\sigma^2$ so that
for all $i \in \mathbb{N}$:
\begin{align*}
 \norm{X_{i}} \leq c,& \; \; \mathbb{P}\text{-almost surely};\\
 \e[1]{ \norm{X_{i}}^{2}} &\leq \sigma^2.
\end{align*} 


Finally, throughout this work, being in the framework of the general Definition~\ref{def:gen_phi_mixing_alter}, we will consider functional classes $\mathcal{C}$ with a semi-norm $C(\cdot)$ satisfying the following assumption.

\textbf{Assumption A3.}
Let, as it was assumed before, $C(f)$ be a semi-norm defined on a subspace $(\mathcal{C},\norm{\cdot}_{\mathcal{C}})$ of real bounded functions $\{f: \mathcal{X} \mapsto \mathbb{R}\}$. For each $s \in \mathcal{B}^{\star}$ define $h_{1,s} : x \mapsto \inner{s,x}$ for each $s \in \mathcal{B}^{\star}$ and $h_{2}: x \mapsto \norm{x}^{2}$, where $\mathcal{B}^{\star}$ is the dual space of $\mathcal{B}$. Define $B(r)$, $B^{\star}(r)$ to be the 
closed
balls of radius $r$ centered in zero in $\mathcal{B}$ and $\mathcal{B}^{\star}$, respectively. 

It is assumed that $h_{1,s} \in \mathcal{C} \text{ for all } s \in \mathcal{B}^{\star}; h_{2} \in \mathcal{C}$, and: 
\begin{align*}
	\sup_{s \in B^{\star}(1)} C(h_{1,s})& \leq C_{1}, \\
	C(h_{2}) & \leq C_{2},
\end{align*}
for some fixed constants $C_{1},C_{2} \in \mathbb{R}_{+}$.

{\bf Example~\ref{example2.1} (continued).} 
For the Lipschitz class $\mathcal{C}_{\Lip}$ considered in Example~\ref{example2.1} we have:
    \[
    \sup_{s \in B^{\star}(1)}C_{\Lip}(h_{1,s}) = 	\sup_{s \in B^{\star}(1)}\norm{h_{1,s}}_{\Lip({B}(c))} = \sup_{\substack{s \in B^{\star}(1) \\ x_{1},x_{2} \in B(c)}}\bigg\{ \frac{ \inner{s, x_{1}-x_{2}}}{\norm{x_{1}-x_{2}}}\bigg\} \leq 1,
\]
and 
\[
C_{\Lip}(h_{2}) = \norm{h_{2}}_{\Lip({B}(c))} = \sup_{x_{1},x_{2} \in B(c)}\bigg\{ \frac{\abs[1]{\norm{x_{1}}^{2}-\norm{x_{2}}^{2}}}{\norm{x_{1}-x_{2}}}\bigg\}
\leq 2c.
\]

\textbf{Example~\ref{example2.2} (continued).} For the $\BV$ functional class $\cC_{\BV}$ considered
in Example~\ref{example2.2}, and $\cX = [-c,c] \subset \mathbb{R}$ we get
(note that in this case $B^{\star}(1)=[-1,1]$ and the functional $h_{1,s}$ is just multiplication by $s$):
\begin{align*}
\sup_{s \in B^{\star}(1)}C_{\BV}(h_{1,s}) &= \sup_{\abs{s} \leq 1}\norm{h_{1,s}}_{ \BV(B(c)) } =  \sup_{\abs{s}\leq 1}\sup_{(x_{0},\ldots,x_{n}) \in \triangle } \sum_{i=1}^{n}\abs{s(x_{i}-x_{i-1})} = 2c.
\end{align*} 
\[
C_{\BV}(h_{2}) = \norm{h_{2}}_{ \BV({B}(c))} =  \sup_{(x_{0},\ldots,x_{n}) \in \triangle } \sum_{i=1}^{n}\abs{x_{i}^{2}-x_{i-1}^{2}} = 2c^{2}.
\]
\subsection{Main result and corollaries}

Our main result is a Bernstein-type inequality for norms of sums of bounded Banach-valued random variables which are generated by some 
centered $\Phi_{\mathcal{C}}-$mixing process. 
We begin with a general bound on the deviations of the norm of $\sum_{i=1}^{n}X_{i}$.
\begin{thm}
  \label{thm:general_result_all}
  Let $\paren{\Omega,\mathcal{F},\mathbb{P}}$ be an arbitrary probability space,  $\paren{\mathcal{B},\norm{\cdot}}$ a Banach space such that Assumption~{\bf A1} holds and 
$\mathcal{X}=B(c)$.
 Let $(X_{i})_{i\geq 1}^{n}$ be
an  $\mathcal{X}$-valued, centered, 
  $\mathcal{C}$-mixing random process on $\paren{\Omega,\mathcal{F},\mathbb{P}}$ such that Assumptions~\textbf{A2,A3} are satisfied. Then for each pair of positive integers $ (\ell,k),$ $\ell\geq 2$, such that $n = \ell k+r, r\in \{0,\cdots ,k-1\}$, and any $\nu >0$, it holds:
\begin{align}
  \begin{aligned}
    \prob{\norm[3]{\frac{1}{n}\sum_{i=1}^{n}X_{i}} \geq 4A_{1}C_{1}\Phi_\mathcal{C}\paren{k} + 4 \sqrt{\frac{B\paren{\sigma^{2} + C_{2}\Phi_{\mathcal{C}}\paren{k}}\nu}{\ell}} + \frac{4c\nu}{3\ell}} \leq 2\exp\paren{-\nu},
  \end{aligned}
\end{align}
	where $B = A_{1}^{2}+ A_{2}$ and the constants $A_{1},A_{2},C_{1},C_{2}$ are given
        by the assumptions.
\end{thm}
Since the choice of $k$ and $\ell$ in the above result is free subject to $k=\big\lfloor \frac{n}{\ell} \big\rfloor$, one can optimize the obtained deviation bound over the choice of
$\ell$ in order to reach the most favorable trade-off between the first term
of order $\Phi_{\mathcal{C}}(\big\lfloor{\frac{n}{\ell}}\big\rfloor)$ which is nondecreasing in $\ell$, and the following "Bernstein-like" terms.
This trade-off is a direct consequence of the so-called blocking technique used
in the proof of the above result: the sample is divided into $k$ blocks of size $\ell$ or $\ell+1$, such that the distances between two neighbor points in a
same block is exactly $k$. The Bernstein-like deviation terms are similar
to the ones found in the i.i.d. case, but with the total sample size $n$
replaced by the block size $\ell$. The terms involving $\Phi_{\cC}$ reflect
the lack of independence inside a block.
This trade-off leads us to the notion of \textit{effective sample size}. For a given $n$ and constants $c,\sigma^{2}$ we define the positive integer number $\ell^{\star}$:
\begin{equation}
\label{df:effective_sample_size}
\ell^{\star} := \max\set[2]{1 \leq \ell \leq n \text{ s.t. } C_1 \Phi_{\mathcal{C}}\paren{\Big\lfloor{\frac{n}{\ell}}\Big\rfloor} \leq \frac{c}{\ell}\vee \frac{\sigma}{\sqrt{\ell}} } \cup \{1\}.
\end{equation}  

Observe that $\ell^\star$ is a function of $n$, but we omit this dependence to simplify notation. The following consequence of Theorem~\ref{thm:general_result_all} is formulated in terms of the \textit{effective sample size}:
\begin{thm}
\label{thm:main_result_gen}
Assume the conditions of Theorem \ref{thm:general_result_all} are satisfied, and 
the effective sample size $\ell^{\star}$ is as given by~\eqref{df:effective_sample_size}.
Then for any $\nu \geq 1$:
\begin{equation}
\label{eq:gen_res02}
	\prob[3]{\norm[2]{\frac{1}{n}\sum_{i=1}^{n}X_{i}} \geq {\frac{\sigma  (4 A_1 + 6\sqrt{B} \sqrt{\nu})}{\sqrt{\ell^{\star}}} + \frac{c(4A_1+ M_1 \nu)}{\ell^{\star}}} } \leq 2\exp(-\nu),
\end{equation}
where 
$M_{1}:= 2+ 2\sqrt{B}(1 + 2 \frac{C_{2}}{C_1 c})$. 

\end{thm}

\textit{Remark.} Lest the reader should wonder at the apparent lack of multiplicative scaling invariance
  of the last result due to the constant $C_2/(C_1c)$ appearing in $M_1$, we stress that the $\cC$-mixing
  assumption is not invariant with respect to rescaling of the value space in general.
  However, in the particular cases of $\tau$- and $\tilde{\phi}$-mixing (Examples~\ref{example2.1}, \ref{example2.2}), the mixing assumption behaves gracefully with respect to scaling: in both cases it can be checked that the compound quantity $C_1 \Phi_{\cC}(.)$ scales linearly with multiplicative rescaling of the space $\cX$,
  so that the effective sample size $\ell^*$ given by~\eqref{df:effective_sample_size} remains invariant,
  while $C_2/(C_1 c)$ remains constant, so that the deviation inequality \eqref{eq:gen_res02} is
  unchanged by multiplicative rescaling, as one would expect.

Furthermore, we can give more explicit rates by lower bounding the effective sample size
in the specific cases of exponentially or polynomially $\mathcal{C}-$mixing processes.

\begin{proposition}
\label{thm:eff_sample_sizes}
For an exponentially $\mathcal{C}-$mixing centered process on $(\Omega,\mathcal{F}, \mathbb{P})$ with rate $\Phi_{\mathcal{C}}(k) :=  \chi\exp(-(\theta k)^{\gamma})$ ($\chi > 0, \theta >0, \gamma >0$), the effective sample size satisfies
\[
\ell^\star \geq 
\Big\lfloor \frac{n}{2}\theta \paren{1 \vee \log \paren{{c^{-1} C_1 \chi \theta} {n}}}^{-{\frac{1}{\gamma}}}\Big\rfloor.\]
%

For a  polynomially $\mathcal{C}-$mixing centered process with rate
 $\Phi_{\cC}(k)= \rho k^{-\gamma}$, the effective sample size satisfies
\[
\ell^\star \geq 
\max \paren[3]{\bigg\lfloor \paren[2]{\frac{\sigma}{C_1 \rho}}^{\frac{2}{2\gamma +1}}\paren[2]{\frac{n}{2}}^{\frac{2\gamma}{2\gamma +1}}\bigg\rfloor,\bigg\lfloor { \paren[2]{\frac{c}{C_1 \rho}}^{\frac{1}{\gamma +1}} \paren[2]{\frac{n}{2}}^{\frac{\gamma}{\gamma +1}}} \bigg\rfloor}.\]
\end{proposition}
%

In the application section, we will use the obtained concentration framework for sums of
Hilbert-space valued random variables. 
In this particular case, we have $A_{1}=1, A_{2} = 1$ and correspondingly $B=2$. Considering the case where the underlying data generating process is $\tau-$mixing (see Example~\ref{example2.1}) we get $C_{1} = 1$ and $C_{2} = 2c$. 
This gives us the following consequence for the concentration of the norm in the case of a process that satisfies  the $\tau-$mixing conditions mentioned in \textbf{Example~\ref{example2.1}.} 

\begin{cor}[Concentration result for Hilbert-valued $\tau-$mixing processes]
	\label{cor:tau_mix}
	Under the assumptions of Theorem~\ref{thm:main_result_gen} with a Hilbert-valued $\tau-$mixing sample $\{X_{i}\}_{i=1}^{n}$, 
	for any $0 \leq \eta \leq \frac{1}{2}$, with probability at least $1 - \eta$ it holds:
	\begin{equation}
	\norm[3]{\frac{1}{n}\sum_{i=1}^{n}X_{i}} \leq \log\paren{\frac{2}{\eta}} \paren{\frac{13\sigma }{\sqrt{\ell^{\star}}} + \frac{21c}{\ell^{\star}}},
	\end{equation}
	where the choice of $\ell^{\star}$ is given by \eqref{df:effective_sample_size}.
\end{cor}

\subsection{Discussion of results}

\label{se:disc}

%

We highlight aspects
in which our results differ from previous work.
We first restrict our attention to the real-valued case ($\cB=\mbr$).
We consider the general type of $\Phi_{\mathcal{C}}-$mixing processes as in \cite{Steinwart:16},
where the authors require the additional assumption on the semi-norm $C(\cdot)$ that
the inequality $C(e^f) \leq \norm{f}_{\infty} C(f)$ should hold for all $f \in \cC$.
Instead, we only pose the assumption that the underlying class $\mathcal{C}$ contains linear forms and the function $x \mapsto \norm{x}^{2}$, plus a.s. boundedness. The reason is that the proof  of the main result essentially relies on the representation of the norm by means of its second order Taylor expansion. This allows us to recover results analogous to \cite{Steinwart:16} (in the sense of the order of the effective sample size) for geometrically $\Phi_{\mathcal{C}}-$mixing processes.
In this case, a broad overview and comparison to existing
literature is given in \cite{Steinwart:16}; we omit reproducing this detailed discussion here and
refer the reader to that work.
As a further contribution with respect to \cite{Steinwart:16}
we derive new results for the exponential
concentration of the sum for polynomially $\Phi_{\mathcal{C}}-$mixing processes. 


In the general Banach-valued case, the norm can be seen as a particular case
of general functionals of the sample. As mentioned in the introduction,
while the literature on concentration of general functionals
in the independent case is flourishing, it is rather scarce under
the setting of weak dependence. In the work \cite{Kontorovich:08}, the  authors obtain general Hoeffding-type concentration inequalities for functionals of the sample satisfying the
bounded difference assumption (Azuma-McDiarmid type setting) under the so-called $\eta-$mixing assumption
(which is related to, but weaker than, $\phi$-mixing). The core proof technique in
our results as well as in \cite{Kontorovich:08} is the martingale difference approach.

Furthermore, in the work \cite{Dedecker:15}, the authors establish a Marcinkiewicz–Zygmund type inequality for dependent random Banach-valued sums under assumptions on the smoothness of the corresponding norm
which are very close to ours. In particular, from Corollary~3.2 in \cite{Dedecker:15}, one can deduce that for a bounded $\tau-$mixing process $(X_{i})_{i\geq 0}$ (see~\textbf{Example~\ref{example2.1}.}) with values in $\mathbb{L}^{q}\paren{\Omega,\mathcal{A},\mathbb{P}}$ for $q \geq 2$, a Hoeffding-type exponential bound holds for sums  with a deviation rate 
of order $ 2cb_{n}\sqrt{\log\paren{e/\delta}/n}$,
where $c$ is as in Assumption~\textbf{A2},  and $ b^{2}_{n} :=1+ \sum_{i=1}^{n}\tau(i)$. 

Comparing to this last deviation bound, an advantage of our results is
that they are of Bernstein- rather than Hoeffding-type, and
valid under a weaker dependence assumption, which includes $\tau-$mixing as specific case. On the other hand, the deviation
scaling in $b_n$ above is better than
ours (this is in particular relevant for polynomial mixing conditions),
which leaves open for future work the question of obtaining a similar
scaling for Bernstein-type deviations under the assumptions we consider.

The strongest assumption we make (besides those concerning the geometry of $\cB$ and the class $\cC$) 
is the a.s. boundedness of the random variable; this assumption was also made in \cite{Pinelis:86} and \cite{Yurinskyi:70} (Bernstein-type inequalities for a Banach-valued independent process, which
is included in the present result) and in \cite{Steinwart:16} (Bernstein-type inequality for weakly dependent
real variables). From a technical point of view, our
current proof significantly relies on that assumption at several key places;
removing this assumption to replace it by a weaker control of moments (as in the classical independent real-valued Bernstein inequality) is a stimulating question.

We now apply the concentration results to the particular case of random variables with values in a separable Hilbert space, and use them for the analysis of statistical properties of
kernel-based algorithms in machine learning which are trained on a dependent sample. This analysis will be the cornerstone of the next section.

\section{Application to statistical learning}
\label{sec:application}

Let $\mathcal{X}$ be a closed ball of a Polish space and $\mathcal{Y}=\mathbb{R}$.
Let as before $(Z_{i})_{i \geq 1}$ be a~stationary
stochastic process over some probability space $\paren{\Omega,\mathcal{F},\mathbb{P}}$ with values in $\mathcal{X}\times \mathcal{Y}$, and define $\nu$ as the common marginal distribution of the $Z_i$s, and $\mu$ as its $X$-marginal. We will also denote $\nu(y|x)$ a
regular conditional probability distribution of $Y_i$ conditional to $X_i$.
In the general framework of learning from examples, the goal is to find a prediction function $f:\mathcal{X} \mapsto \mathcal{Y}$ such that for a new pair $(X,Y) \sim \nu$, the value $f(X)$ is a good predictor for~$Y$. Let  $\bz := \{ x_{i},y_{i}\}_{i=1}^{n} \in \paren{\mathcal{X} \times \mathcal{Y}}^{n}$ be the observed training sample from the $n$ first coordinates of the process $(Z_{i})_{i\geq 1}$, and  $f_{\bz}$ be an estimated prediction function
belonging to some model class $\mathcal{H}$.
We will assume $(Z_{i})_{i\geq 1}$ to be a $\tau-$mixing stationary process (as in Example~\ref{example2.1}) on $(\Omega,\mathcal{F},\mathbb{P})$. 
We consider the least squares regression problem where the goal is to minimize the averaged squared loss $\mathcal{E}(f) := \ee[1]{\nu}{\paren{f(X)-Y}^{2}}$. Equivalently, we  want to find $f_{\bz}$ that approximates the regression function $f_{\nu}(x) = \e{Y|X=x}$ 
well in the sense of being close to optimal risk $\mathcal{E}(f)$ over the considered model class.

\subsection{Learning and regularization using reproducing kernels}

We investigate statistical learning methods
based on reproducing kernel Hilbert space regularization, i.e., we consider as a model class
a separable real reproducing kernel Hilbert space (RKHS) $\mathcal{H} = \mathcal{H}_{k} \subset L^{2}(X,\mu)$ which is induced by a measurable kernel $k$ over $\cX^2$.
In the next pages we recall the setting and notation used in this framework and therefore
reiterate in shortened form some of the corresponding content of
\cite{Bauer:09,Blanchard:17}; for more details see also \cite{Caponetto:07,Rosasco:10}.

We assume the kernel to be bounded by a positive constant $\kappa=1$, i.e. $\sup_{x \in \mathcal{X}}\sqrt{k(x,x)} \leq 1$. 
This implies that any $f \in \mathcal{H}_{k}$ is measurable and bounded in the supremum norm. As $\mathcal{H}_{k}$ is a subset of $L^{2}(\mathcal{X},\mu)$, let $S_{k} : \mathcal{H}_{k} \mapsto L^{2}(\mathcal{X},\mu)$ be the inclusion operator;
and  $S^{\star}_{k} : {L}^{2}(\mathcal{X},\mu) \mapsto \mathcal{H}_{k}$ its adjoint. Using the definition of the adjoint, $S^{\star}_{k}$ can be written as: 
\begin{align*}
\begin{aligned}
S^{\star}_{k}g= E_{\mu}\brac[1]{g(X)k_{X}} := \int_{\cX}k_{x}g(x)\mu(dx),
\end{aligned}
\end{align*}
where $k_{x}$ is the element $k(x,\cdot) \in \mathcal{H}_{k}$.
Similarly, for $x \in \mathcal{X}$ denote 
$T_{x} = k_{x} \otimes k_{x}^{\star}$, then the covariance $T : \mathcal{H}_{k} \mapsto \mathcal{H}_{k}$ 
has the following representation:
\begin{align*}
\begin{aligned}
T := E_{\mu}\big[T_{X}\big]  = \int_{\cX}\langle \cdot,k_{x} \rangle_{\mathcal{H}_{k}}k_{x}\mu(dx),
\end{aligned}
\end{align*}
where the last integrals are understood in the Bochner sense and $T, T_{X}$ both are self-adjoint and trace-class.
We use the notation $\HS(\mathcal{H}_{k})$ for the space of Hilbert-Schmidt operators 
over $\mathcal{H}_{k}$, so that the former implies that $T,T_{X} \in \HS(\mathcal{H}_{k})$.

We obtain the empirical analogues of the operators $T$, $S_{k}$, $S^{\star}_{k}$ by replacing the measure $\mu$ with its empirical counterpart $\hat{\mu}_{\bx} = \frac{1}{n}\sum_{i=1}^{n}\delta_{x_{i}}$,
where ${L}^{2}(\mathcal{X},\hat{\mu}_{\bx})$ is identified as $\mathbb{R}^n$ endowed with the standard scalar product
rescaled by $n^{-1}$.
We define the following empirical operators:

\begin{align*}
S_{\bx} : & \mathcal{H}_{k} \mapsto \mathbb{R}^{n}, & \text{    } (S_{\bx}f)_{j} & = \langle f,k_{x_{j}} \rangle, \\
S ^{\star}_{\bx} :& \mathbb{R}^{n} \mapsto \mathcal{H}_{k}, & \text{    } S_{\bx}^{\star}\by & = \frac{1}{n}\sum_{j=1}^{n}y_{j}k_{x_{j}}, \\
T_{\bx}:= S^{\star}_{\bx}S_{\bx} :& \mathcal{H}_{k} \mapsto \mathcal{H}_{k}, & \text{    } T_{\bx} f &= \frac{1}{n}\sum_{j=1}^{n}k_{x_{j}} \inner{k_{x_{j}},f},
\end{align*}
where we used the notation $\by=(y_{1},\ldots,y_{n}) \in \mathbb{R}^{n}$. 


 We now specify classes of distributions which correspond to a certain regularity of the
 learning problem in relation to the RKHS $\cH_k$, and on which we will aim at establishing error bounds.
We start with the following assumption on the underlying distribution $\nu$ and the corresponding
regression function $f_{\nu}$. 

\textbf{Assumption B1.} There exist $0 < R\leq 1,\Sigma>0$ such that the distribution $\nu$ belongs to the set $\cD(R,\Sigma)$ of distributions satisfying:
\begin{itemize}
\item[i)] $\abs{Y}\leq R$, $\nu$-almost surely.
\item[ii)] 
  The regression function $f_\nu$ belongs to the RKHS $\cH_k$, i.e. for $\mu$-almost all $x\in \cX$ it holds
	\begin{align*}
			\e[1]{Y|X=x} := \int_{y \in \mathcal{Y}}yd\nu(y|x) = f_{\nu}(x),
			\; f_{\nu} \in \mathcal{H}_{k}. 
	\end{align*} 
	\item[iii)] For $\mu-$almost all $x$: 
	\begin{align*}
	\begin{aligned}
	\mathrm{Var}\paren[1]{Y|X=x} = \int_{y \in \mathcal{Y}}\paren{y-f_{\nu}(x)}^{2}d\nu(y|x) \leq \Sigma^{2}.
	\end{aligned}
	\end{align*}
%
\end{itemize}

Point (i) of the assumption ensures that we can assume $\cY=[-R,R]$ without loss of generality. The two next assumptions are: a decay rate condition for the discrete spectrum $(\zeta_{i})_{i \geq 1}$
 (ordered in decreasing order) of the 
 covariance operator $T$, and a so-called H\"older source condition (see e.g. \cite{DeVito:06}) that describes the smoothness of the regression function $f_{\nu}$. Denoting $\mathcal{P}$ to be the set of all probability distributions on $\mathcal{X}$; we will thus assume
 that the $X$-marginal distribution $\mu$ belongs to
 \[	\mathcal{P}^{<}(b,\beta) := \set[1]{ \mu \in \mathcal{P}: \zeta_{j} \leq \beta j^{-b}, \forall j \geq 1 };
 \]
secondly, we assume that $f_{\nu} \in \Omega({r,D})$, where
\begin{equation}
\label{eq:hoelder_source_cond}
\Omega({r,D}) = \set[1]{f \in \mathcal{H}_{k}| f = {T}^{r}g, \norm{g}_{\mathcal{H}_{k}}\leq  D},
\end{equation}
which in the inverse problems literature is called the standard H\"older source condition for the linear embedding problem. 
Joining all assumptions, we consider the following class of marginal generating distributions:
\begin{equation}
  \label{eq:defModel}
  \mathcal{M}(R,\Sigma,r,D,\beta,b):=\\ \set[1]{ \nu(dx,dy) = \nu(dy|x)\mu(dx) : \nu \in \cD(R,\Sigma),
  \mu \in \mathcal{P}^{<}(b,\beta), f_{\nu} \in \Omega({r,D}) }.
\end{equation}

For estimation of the target regression function $f_{\nu}$, we consider the following class of kernel spectral~regularization methods:
\begin{equation}
\label{eq:gen_alg_RKHS}
f_{\bz,\lambda} = F_{\lambda}({T}_{\bx}){S}^{\star}_{\bx}\by,
\end{equation}
where ${F}_{\lambda} : [0,1] \mapsto \mathbb{R}$ is a family of functions.
The expression $F_{\lambda}({T}_{\bx})$ is to be understood in the usual sense of
(compact, selfadjoint) functional calculus on operators. The family $(F_\lambda)_{\lambda \in [0,1]}$
defines the regularization method (which we also call regularization function), depending on the parameter $\lambda \in (0,1]$, and for which the following conditions hold:
\begin{itemize}
	\item[i)] There exists a constant $B < \infty$ such that, for any $0 < \lambda \leq 1$:

	$$\sup_{t \in (0,1]}\abs[1]{tF_{\lambda}(t)} \leq B.$$
	\item[ii)] There exists a constant $E < \infty$ such that 
	$$\sup_{t \in (0,1]}\abs[1]{tF_{\lambda}(t)} \leq E / \lambda. $$ 
	\item[iii)] There exists a constant $\gamma_{0}$ such that the \textit{residual} $r_{\lambda}(t) := 1-F_{\lambda}(t)t$ is uniformly bounded, i.e.
	$$\sup_{t \in (0,1]} \abs[1]{r_{\lambda}(t)} \leq \gamma_{0}.$$
	\item[iv)] For some positive constant $\gamma_{q}$ there exists a maximal $q$, which is called \textit{the qualification of the regularization} such that
	$$ \sup_{t \in (0,1]}\abs[1]{r_{\lambda}(t) t^{q}} \leq \gamma_{q}\lambda^{q}.$$
\end{itemize}
The above conditions are standard in the framework of inverse problems and in asymptotic framework are sufficient (see \cite{Bauer:09}) in order to obtain consistent learning algorithms in case of independent examples. Many known regularization procedures (including Tikhonov regularization, spectral cut-off, Landweber iteration) may be obtained as special cases via appropriate choice of the regularization function $F_\lambda$ and
satisfy conditions $i)-iv)$ for appropriate parameters.
 We refer the reader to \cite{Engl:96}, \cite{Bauer:09} and \cite{Rosasco:10} for a variety of different examples as well as the discussion in the context of learning from independent examples.

 \subsection{Learning from a  $\tau-$mixing sample}
We now restrict our attention to the case of $\tau-$mixing processes, i.e. to those which are generated by the Lipschitz seminorm $C_{\Lip}(\cdot)$ from Example~\ref{example2.1}.
Obtaining probabilistic results for Hilbert-valued estimators
(analogous in spirit to those in \cite{Blanchard:17}), we derive upper bounds on the estimation error of $f_{\nu}$ 
by regularized kernel learning estimators $f_{\bz,\lambda}$,
in the case of learning from $\tau$-mixing samples, assuming a polynomial spectrum decay rate
of the covariance operator $T$,
and for a certain range of norms.

A key technical tool used in previous works for the analysis of the i.i.d. case
(see \cite{Bauer:09,Blanchard:17})
 is a quantitative statement for the concentration of the centered
(and possibly suitably rescaled)
Hilbert-space valued variables $(S^*_\bx \by-T_\bx f_\nu)$ and $(T_{\bx}-T)$ around~0.
Observe that these variables are empirical sums (of elements $k_{x_i}(y_i-f_\nu(x_i)) \in \cH_k$ and
$(k_{x_i} \otimes k_{x_i}^\star -T)\in \HS(\cH_k)$, respectively). Thus, a very natural way to proceed
in the analysis is to use the concentration results established in Section~\ref{sec:main_results} for Hilbert spaces as replacement for their i.i.d. analogues, and for other steps follow the proof strategy of those earlier works.

Assuming the sample $\bz = \{x_{i},y_{i}\}_{i=1}^{n}$ is a realization from  a $\tau-$mixing process $(Z_{i})_{j\geq 1}$, in order to apply the concentration inequality from Section~\ref{sec:main_results}, we should ensure that the corresponding Hilbert-valued quantities
are forming a $\tau$-mixing sequence themselves. As pointed out earlier, the
$\tau$-mixing property is obviously preserved (up to constant) via a Lipschitz mapping.
Lemma~\ref{lem:lipschitz_property} in Appendix~\ref{app} establishes
this Lipschitz property for the kernel maps under mild assumptions (uniformly bounded mixed second derivative
of the kernel).
Using the inequality from Corollary~\ref{cor:tau_mix},
in Lemma~\ref{lem:operator_deviation}  we  obtain high probability inequalities for deviations of the corresponding random elements. The proof of the lemma can be found in Appendix~\ref{app:02}.
To simplify the exposition, 
we specify the results for the cases of either exponentially or polynomially mixing process. Further extensions are possible using the same general proof scheme as a blueprint, described in Appendix~\ref{app:02}, together with the result of Theorem \ref{thm:main_result_gen} on the
effective sample size.

\begin{lemma}
  \label{lem:operator_deviation}
  
  Let $\mathcal{X}$, $\mathcal{Y}=[-R,R]$ and $\mathcal{H}_{k}$ be as defined before.
  Assume that the kernel $k$ satisfies $\sup_{x \in \mathcal{X}}\sqrt{k(x,x)} \leq 1$ and
  admits a mixed partial derivative $\partial_{1,2} k : \mathcal{X}\times \mathcal{X} \mapsto {\mathbb{R}}$ which is uniformly bounded by some positive constant $K$. Let $(Z_{j}=(X_{j},Y_{j}))_{j \geq 1}$ be a $\tau-$mixing process with rate $\tau(k)$, satisfying Assumption~{\bf B1} and such that $\norm{f_\nu}\leq D$.
  
  For any $\eta \in (0,1/2]$ 
%
  the probability of each one of the following events is at least $1-\eta$:
  \begin{align}
    \begin{aligned}
      \label{eq:main_tools}
      \norm[1]{{T}_{\bx}f_{\nu} - {S}_{\bx}^{\star}y}  & \leq 21\log\paren[1]{2\eta^{-1}}
      \paren[3]{{\frac{\Sigma^{}}{\sqrt{\ell_{1}^{}}}}+ \frac{2R}{\ell_{1}^{}}};\\
      \norm[1]{\paren{{T}+\lambda}^{-\frac{1}{2}}\paren{{T}_{\bx}f_{\nu} - {S}_{\bx}^{\star}y}}  & \leq 21\log\paren{2{\eta}^{-1}} 	 \paren[3]{\frac{\Sigma\sqrt{\mathcal{N}(\lambda)}}{\sqrt{\ell_{2}^{}}} + \frac{2R}{\sqrt{\lambda}\ell_{2}^{}}};\\
      \norm[1]{\paren{{T} + \lambda}^{-1/2}\paren{{T}-{T_{\bx}}}} & \leq 21\log\paren{2{\eta}^{-1}}  \paren[3]{ \frac{\sqrt{\mathcal{N}\paren{\lambda}}}{\sqrt{\ell_{3}}} +
        \frac{2}{\sqrt{\lambda}\ell_{3}}}  ; \\
      \norm[1]{{T} - {T_{\bx}}} & \leq 42 \frac{\log\paren{2\eta^{-1}}}{\sqrt{\ell_{4}}},
    \end{aligned}
  \end{align}	
where the quantity  $\mathcal{N}(\lambda):=Tr\paren[1]{\paren{{T}+\lambda}^{-1}{T}_{}}$ is the
  so-called {\em effective dimension};
  $\ell_{1},\ell_{2},\ell_{3},\ell_{4}$  are in each case suitable bounds on the effective sample size. For exponentially and polynomially $\tau-$mixing  rates,  corresponding bounds for effective sample
  sizes are given in Table~\ref{tab:table1}.

\begin{table}[h!]
  \centering
  \caption{Bounds on effective samples sizes for \eqref{eq:main_tools}. Put $C:=3\max(1,KR,KD)$ here.}
  \label{tab:table1}
  \begin{tabular}{ccc}
    \toprule
    &  $\tau(k) = \chi \exp\paren{-(\theta k)^{\gamma}}$ &  $\tau(k) = \rho{ k^{-\gamma}}$ \\
    \midrule
    $\ell_{1}$ & \multirow{2}{*}{ \raisebox{-4mm}{$  \;\; \left\lfloor \frac{n \theta}{2\paren{1\vee \log\paren{n\frac{C\chi\theta}{2
					R}}}^\frac{1}{\gamma}}\right\rfloor$}}
                                                         &
                                                           $\bigg\lfloor{\paren{\frac{\Sigma}{C\rho}}^{\frac{2}{2\gamma +1}}\paren{\frac{n}{2}}^{\frac{2\gamma}{2\gamma +1}}}\bigg\rfloor$\\
    $\ell_{2}$ &  
                                                         & $\bigg\lfloor{\paren[2]{\frac{\Sigma \sqrt{{\lambda \mathcal{N}(\lambda)}}}{C\rho}}^{\frac{2}{2\gamma +1}}\paren{\frac{n}{2}}^{\frac{2\gamma}{2\gamma +1}}}\bigg\rfloor$\\
    $\ell_{3}$ &  \multirow{2}{*}{$\Bigg\lfloor{\frac{n\theta}{2\paren{1 \vee \log \paren{{n K\theta \chi}}}^{\frac{1}{\gamma}}}}\Bigg\rfloor$} & $\bigg\lfloor{\paren[2]{\frac{\sqrt{\lambda \mathcal{N}(\lambda)}}{2K\rho}}^{\frac{2}{2\gamma +1}}\paren{\frac{n}{2}}^{\frac{2\gamma}{2\gamma +1}}}\bigg\rfloor$\\
    $\ell_{4}$ &  
                                                         & $\bigg\lfloor{\paren{\frac{1}{K\rho}}^{\frac{2}{2\gamma+1}}\paren{\frac{n}{2}}^{\frac{2\gamma}{2\gamma +1}}}\bigg\rfloor$\\
    \bottomrule
  \end{tabular}
\end{table}

\end{lemma}
\textit{Remark}. The first inequality will not be used in the statistical analysis to follow
and is presented here for completeness.

Armed with the above probabilistic results, we derive upper bounds for 
the errors of estimation of $f_{\nu}$ by means
 of the general regularized kernel learning estimators \eqref{eq:gen_alg_RKHS}. 
The main tool is the following lemma, giving a high probability inequality on the deviation of the estimation error. The gist of this result and of its proof is to follow the approach
of \cite{Blanchard:17}, wherein the sample size in the i.i.d. case is replaced by the effective
sample size, the rest of the argument being essentially the same.

\begin{lemma}
  \label{lem:error_bound_main}
  Consider the same assumptions as in Lemma~\ref{lem:operator_deviation}.
  Assume that $f_\nu \in \Omega(r,D)$ (defined by~\eqref{eq:hoelder_source_cond}) for some positive numbers $r,D$.
  Also, let $f_{\bz,\lambda}$ be the regularized estimator as in \eqref{eq:gen_alg_RKHS}, with a regularization satisfying conditions (i)-(iv) with qualification $q\geq r+s$.
Fix numbers $\eta \in (0,1]$ and $\lambda \in (0,1]$ and denote:
	\begin{align*}
	\overline{\gamma} := \max (\gamma_{0},\gamma_{q}), \qquad
	\ell_{0} := {2500 \lambda^{-1}\max(\mathcal{N}(\lambda),1)\log^{2}\paren{\frac{8}{\eta}}},
	\end{align*}
        where we recall that $\gamma_{0},\gamma_{q}$ are the constants from conditions iii)-iv).
	 
	
	Then with probability at least $1-\eta$, the inequality
	\begin{multline}
	\norm{{T}^{s}\paren{f_{\mathcal{H}_{k}} - f_{\bz,\lambda}}}_{\mathcal{H}_{k}}  \\
	 \leq  C_{r,s,B,E,
		\overline{\gamma}}
	\log(8\eta^{-1})
	\lambda^{s} \paren{ 
		D\paren{\lambda^{r} + \frac{1}{\sqrt{\ell^{'}}}}+
		\paren{\frac{R}{\ell^{'}\lambda} + \sqrt{\frac{\Sigma^{2}\mathcal{N}(\lambda)}{{\lambda} \ell^{'}}}}}
	\end{multline}
	holds with 
	$\ell^{'}=\min\{\ell_{2},\ell_{3},\ell_{4}\}$, provided that $\ell^{'} \geq \ell_{0}$ and all $\ell_{i}$ are as in Table~\ref{tab:table1}.  

	 
\end{lemma}
	

We remark that the choice $s=0$ 
corresponds to the estimation error in the space $\mathcal{H}_{k}$, whereas $s=\frac{1}{2}$ corresponds to the prediction error in the space ${L}^{2}(\mathcal{X},\mu)$.

Finally, we establish asymptotic error bounds for the family of regularized estimators of the type~\eqref{eq:gen_alg_RKHS}, when learning from a stationary $\tau-$mixing sequence whose marginal distribution
belongs to the class $\mathcal{M}(R,\Sigma,r,D,\beta,b)$, under appropriate choice of the regularization parameter sequence $\lambda_{n}$. To simplify somewhat expressions, we will assume from now on, without loss of
generality, that $D \geq R \geq 1$ holds.
We separate the analysis between the cases of 
exponentially and polynomially $\tau-$mixing processes.

For an exponentially $\tau-$mixing process $\paren{X_{i},Y_{i}}_{i \geq 1}$ with mixing rate $\tau(k)=\chi\exp\paren{-(\theta k)^{\gamma}}$,  we set: 
\begin{align}
  \label{eq:cor_rate}
  \ell^{'}_{g}(n) := \left\lfloor \frac{n\theta}{2\paren{1 \vee \log\paren{{3nKD\chi\theta R^{-1}}{}}}^{\frac{1}{\gamma}}} \right\rfloor, \qquad
  \lambda_{n} := \min \paren{ \paren{\frac{\Sigma^{2} 
}{D^2\ell'_g}}^{\frac{b}{2br+b+1}},1};
\end{align}
we observe in particular (by straightforward calculation, using the fact that $D \geq R \geq 1$) that the constraint $\ell^{'}_{g} \leq \min \{\ell_{2},\ell_{3},\ell_{4} \}$ is fulfilled.
We are then able to formulate the next statement:
  \begin{thm}
   \label{cor:balance_upper_bound}
   Assume the data distribution $\nu$ belongs to the class $\mathcal{M}(R,\Sigma,r,D,\beta,b)$, and $f_{\bz,\lambda_{n}}$ is a kernel
   spectral regularization estimator~\eqref{eq:gen_alg_RKHS}
   with qualification $q\geq r+s$, where $\lambda_{n}$ is given by \eqref{eq:cor_rate}. 
   Fix some $\eta \in (0,1]$.
 	
%
 	Then there exists $n_0$ (depending on all the model parameters and of $\eta$) such that for $n \geq n_{0}$, it holds with probability at least $1-\eta$: 
 	 	\begin{align}
 			\norm{T^{s}\paren{f_{\nu} - f_{\bz,\lambda_{n}}}}_{\mathcal{H}_{k}} \leq
 			C_{\star}
 			 \log(8\eta^{-1})D\paren{\frac{\Sigma}{D\sqrt{\ell^{'}_{g}}}}^{\frac{2b(r+s)}{2br+b+1}},
 	\end{align}
 	where $C_{\star} := C_{r,s,B,E,\overline{\gamma},b,\beta}$ is a
        factor depending on the regularization function and model parameters (other than $D,R,\Sigma$).
        
%
%
\end{thm}
We establish an analogous result for a polynomially $\tau-$mixing process $\paren{X_{i},Y_{i}}_{i \geq 1}$ with mixing rate $\tau(k)=\rho{ k^{-\gamma}}$, this time without precisely tracking the effects of the constants $\paren{\Sigma,D}$.
We also only consider the case of a two-sided controlled spectrum
\[	\mathcal{P}^{\lessgtr}(b,\beta_-,\beta_+) := \set[1]{ \mu \in \mathcal{P}: \beta_- j^{-b} \leq \zeta_{j} \leq \beta_+ j^{-b}, \forall j \geq 1 },
\]
and the model $\wt{\cM}(R,\Sigma,r,D,\beta_{\pm},b)$ defined as in~\eqref{eq:defModel} with $\cP^<$ replaced by
$\cP^{\lessgtr}$. The technical reason for adding an assumption of lower bounded spectrum is that it implies
a lower bound on the effective dimension $\cN(\lambda)$, and in turn a lower bound on the effective sample size,
which involves the effective dimension in the polynomial mixing case (see Table~\ref{tab:table1}).

We consider the following parameter sequence:
\begin{align}
\label{eq:cor_rate_pol}
\lambda_{n} := n^{-\frac{b}{2br +b + 1 + b(r+1)\gamma^{-1}}}.
\end{align}
Similarly to the case of exponential mixing, we use Lemma~\ref{lem:error_bound_main} with the choice
$\ell_{p}^{'} = \cO\paren{\paren{\lambda_{n} \mathcal{N}\paren{\lambda_n{}}}^{\frac{2}{2\gamma +1}} n^{\frac{2\gamma}{2\gamma +1}}}$, which depends on the regularization and on the effective dimension; by arguments similar to Theorem~\ref{cor:balance_upper_bound} we then obtain:
\begin{thm}
  \label{cor:balance_upper_bound_pol}
     Assume the data  distribution $\nu$ belongs to the class $\wt{\cM}(R,\Sigma,r,D,\beta_\pm,b)$, and $f_{\bz,\lambda_{n}}$ is a kernel
   spectral regularization estimator~\eqref{eq:gen_alg_RKHS}
   with qualification $q\geq r+s$, where $\lambda_{n}$ is given by~\eqref{eq:cor_rate_pol}. For any fixed $\eta \in [0,1]$ and all $n > n_{0}$
   (where $n_{0}$ is such that  $\log(8\eta^{-1}) \leq C_\bigtriangleup^{'} n_{0}^{\frac{br}{2br+b+1 + b(r+1)\gamma^{-1}}}$,
   we have with probability at least $1-\eta$:
	\begin{align}
	\norm{T^{s}\paren{f_{\nu} - f_{\bz,\lambda_{n}}}} \leq C_{\bigtriangleup}\log(8\eta^{-1})n^{-\frac{b(r+s)}{2br+b+1+b(r+1)\gamma^{-1}}},
	\end{align}
	where $C_{\bigtriangleup},C'_\bigtriangleup$
        are factors depending on the regularization and smoothness parameters
        of the model $(R,\Sigma,D,r,s,B,E,\overline{\gamma},b,\beta)$.
  

%
%
\end{thm}

Let us briefly discuss the upper bounds for the risk of the general regularization methods, described in Theorems~\ref{cor:balance_upper_bound} and~\ref{cor:balance_upper_bound_pol}. Asymptotic in nature, those results are based on the concentration inequality~\eqref{thm:main_result_gen}, which allows the control of an error on the exponential scale.
Comparing the result of Theorem~\ref{cor:balance_upper_bound} to risk bounds obtained for i.i.d. scenario (e.g. in~\cite{Blanchard:17}), we observe that in the case of an exponentially mixing process the upper bounds are optimal up to a logarithmic factor, while in the case of a polynomially mixing process with exponent $\gamma$, the rate is degraded by a term depending on $\gamma$, that vanishes as $\gamma\rightarrow \infty$, as one would expect
(whether this rate is optimal in the polynomial mixing case remains unclear). In the case of exponential
mixing, since we describe the explicit dependence of the sequence $\lambda_{n}$ on $\Sigma$ and $D$, further analysis can be conducted exploring other regimes in which either $\Sigma$ or $D$ may depend on $n$.

\subsection{Conclusions and perspectives}

The results of Theorem~\ref{cor:balance_upper_bound} and \ref{cor:balance_upper_bound_pol} are stated in the somewhat standard framework of regularized learning schemes where the
estimator takes values in a Hilbert space, which is generated by some reproducing kernel $k$. 

Since the concentration results of Theorems~\ref{thm:general_result_all} and~\ref{thm:main_result_gen} are valid in the more general case of complete normed spaces which satisfy the smoothness assumption~\textbf{A1}, a natural extension is to consider the setting of statistical learning whose estimators are prediction
functions belonging to a certain functional Banach space.
A variety of such Banach-valued learning schemes (using a corresponding
Banach norm regularization) and  theoretical justification of their validity have been proposed by numerous authors over the years.
Seminal works \cite{Benett:00}, \cite{TZhang:02} have introduced extensions
of convex risk regularization principles to involve a Banach norm regularizer
terms. Further efforts have developed the mathematical foundations of
such methods, in particular concerning properties of so-called evaluation Banach spaces or reproducing kernel Banach spaces and
generalizations of the representer theorem \cite{Canu:03},
\cite{Hein:03},  \cite{Zhang:09}
as well as universal approximation properties of such spaces
\cite{Pontil:04}.






Furthermore, such approaches have given rise to numerous developments in
recent research, for example
extension to the vector-valued kernel setting \cite{Zhang:13} with application to multi-task learning,  the notions of orthomonotonicy, which leads to a generalization of representation theorems \cite{Argyriou:14}, or
combination of these approaches with the kernel mean embedding principle
\cite{Srietal:11}.



Concerning statistical properties  (such as consistency, learning rates and generalization upper bounds for the risk)
of Banach valued learning algorithms, these were also investigated in
\cite{Hein:03}, \cite{Steinwart:09}, \cite{Song:11}
\cite{Villa:14}, albeit only in the case of independent training data. 

Following \cite{Villa:14}, \cite{Zhang:09}, as a direction for future work, one can investigate the geometrical properties of the underlying Banach space norm so as to ensure the possibility of learning in the normed space on the one hand, and to satisfy the smoothness assumption \textbf{A1} on the other.
In such a situation the concentration results presented in this paper will
apply and have the potential to provide a pivotal tool in the analysis of such
schemes for weakly dependent data.



%
%
%

\appendix

\section{Proofs of results in Section~\ref{sec:main_results}}\label{app}
We need the following auxiliary results to complete the proof of Theorem~\ref{thm:main_result_gen}. We will use repeatedly the shorthand notation
$\pfun(x) := e^x - x -  1$.

\begin{lemma}
\label{lem:main_bound}
Assume that $(X_{i})_{i \geq 1}$ is a $\mathcal{C}-$mixing stochastic process {with values in the closed subset $\mathcal{X} = B(c)$ of the separable Banach space $(\mathcal{B},\norm{\cdot})$},
 such that Assumptions~\textbf{A1}, \textbf{A2}, \textbf{A3} hold. Let furthermore $(i_{1},\ldots,i_{k})$ be a $k$-tuple of non-negative integers, such that $i_{1} < i_{2} \ldots < i_{k}$, $\lambda \geq 0$ and $\tilde{S}_{k} := X_{i_{1}} + X_{i_{2}}+ \ldots + X_{i_{k}}$. Then the following holds: 
\begin{align*}
\e[1]{\exp\paren[1]{\lambda \norm[1]{\tilde{S}_{k}} }}
\leq 2 \paren{1+B{\sigma^{2}}\frac{\pfun(\lambda c)}{c^2}
}
  \prod_{j=1}^{k-1}\paren{1+p(d_j,\lambda)}\,,
\end{align*} 
where $p(k,\lambda) := \lambda \tilde{A}_{1}\Phi_{\mathcal{C}}(k) + B \paren{C_{2}\Phi_{\mathcal{C}}(k)+ \sigma^2} {\frac{
  \pfun(\lambda c)}{c^2}}$, 
$d_{j} := i_{j} - i_{j-1}$ for all $j\geq 2$, 
and $B := A_{1}^2+A_{2}$, $\tilde{A}_{1} = C_{1}A_{1}$.
\end{lemma}

\begin{lemma}
\label{lem:laplace_bound}
Assume that  $(X_{i})_{i=1}^{n}$ is a random sample from a 
$\mathcal{X}$-valued centered $\Phi_{\mathcal{C}}-$mixing process, such that Assumptions~{\bf A1, A2, A3} hold. For $n = \ell k+r$, where $\ell,k >1$ are some integers and $r \in \{0,1,\ldots,k-1\}$, and any 
$\lambda\geq 0$ we have: 
\begin{equation}
  \label{eq:lem2}
  \e[2]{\exp\paren[2]{\lambda \norm[2]{\frac{1}{n}\sum_{i=1}^n X_i} }}
  \leq 2 \exp\paren{ \frac{B}{c^{2}}\paren[2]{(\ell+1)\sigma^2 + C_2 \ell \Phi_{\cC}(k)}
    \pfun\paren{\frac{\lambda c}{\ell}}
    + \lambda \tilde{A}_1 \Phi_{\cC}(k)}\,.
\end{equation}
where $\tilde{A}_1$ and $B$ are defined as in Lemma~\ref{lem:main_bound}.
\end{lemma}

\begin{lemma}
\label{lem:prob_bound}
If all the conditions of Lemma~\ref{lem:laplace_bound} hold then the following (exponential) inequality holds:
\[
  \mathbb{P}\paren[2]{\frac{1}{n}\norm[2]{\sum_{i=1}^n X_i} \geq t }
\leq   2\exp\paren{ - \frac{\ell\paren{t^{2} - 4\tilde{m}t}}{4 \paren{\frac{tc}{3}+\tilde{\sigma}^{2}B}}},
\]
where $\tilde{m} := \tilde{{A}_1} \Phi_{\cC}(k)$ and $\tilde{\sigma}^2 := \sigma^2 + C_2 \Phi_{\cC}(k)$. 

Alternatively, for any $\nu > 0$ this can be written as: 
\begin{align*}
 \probb[3]{}{\norm[2]{\frac{1}{n}\sum_{i=1}^{n}X_{i}}  \geq 4 \tilde{m} + 4\sqrt{\frac{B\tilde{\sigma}^{2}\nu}{\ell}}+ \frac{4}{3}\frac{c\nu}{\ell}}\leq 2\exp(-\nu).
\end{align*}
\end{lemma}

\begin{proof}[Proof of Lemma~\ref{lem:main_bound}]

The backbone of the proof follows the technical approach of \cite{Pinelis:92}.
Use as a first step $\e[0]{\exp\paren[0]{\lambda \norm[0]{\tilde{S}_{k}}}} \leq 2\e[0]{\cosh\paren[0]{\lambda \norm[0]{\tilde{S}_{k}}}}$. The next step consists in bounding iteratively, going from $\tilde{S}_{k}$ to
$\tilde{S}_{k-1}$ by conditional expectation. To this end, we first need some (deterministic) bounds
relating $\cosh(\lambda \norm{s+x})$ to $\cosh(\lambda \norm{s})$.


Let $s,x$ be elements of $\mathcal{X}$.
Introduce the following functions for $t\in [0,1]$:
\[
f(t) := \cosh(\lambda h(t))\,,\qquad  h(t):= {\norm{s + tx}}\,.
\]
For any $t \in [0,1]$ such that $h(t)\neq 0$, it holds
\begin{equation}
  \label{eq:fprime}
f'(t)   =
\lambda \sinh(\lambda h(t)) h'(t) 
= \lambda  \sinh(\lambda h(t)) \inner[1]{D_{s+tx}\norm{\cdot},x}\,.
\end{equation}
If for some $t_0$, it holds $h(t_0)=0$, then $h$ itself may not be differentiable in $t_0$, however $f'(t)$ exists,
and is equal to 0, in this case. Namely, if $x=0$ then $h$ must be
identically zero, otherwise $h(t)\neq 0$ for $t\neq t_0$, and
\eqref{eq:fprime} holds for any $t \neq t_0$, implying by Assumption~\textbf{A1} $\abs{f'(t)} \leq A_1 \lambda \norm{x}
\sinh(\lambda h(t))$;
this implies differentiability in $t_0$ 
since the limit of the derivative 
exists (and is equal to 0) as $t \rightarrow t_0$, and the function $f(t)$  is continuous.
Similarly, for any $t\in[0,1]$ with $h(t)\neq 0$, and using Assumption~\textbf{A1}:
 \begin{align*}
   f''(t) & = 
   \lambda^2 \cosh(\lambda h(t)) h'(t)^2 + \lambda \sinh(\lambda h(t)) h''(t) \\
   &= \lambda^2 \cosh(\lambda h(t)) \inner[1]{D_{s+tx}\norm{\cdot},x}^2 + \lambda \sinh(\lambda h(t))
     \delta_{x,x} \paren[1]{\norm[1]{s+tx}}\\
& \leq A_1^2 \lambda^2 \norm{x}^2 \cosh(\lambda h(t)) + A_2 \lambda \norm{x}^2  \frac{\sinh(\lambda h(t))}{h(t)}\\
& \leq \lambda^2 \norm{x}^2 B  \cosh(\lambda h(t))\,,
 \end{align*}
 where we have used $\sinh(x) \leq x \cosh(x)$. We conclude that $f'(t)$ is absolutely continuous: unless
 $h(t)$ is identically 0, there exists at most a single point $t_0 \in [0,1]$ where $h(t_0)=0$ and where
 $f'$ may not be differentiable. We can therefore use the Taylor expansion:
 \begin{equation}
   \label{eq:taylorf}
 f(1) = f(0) + f'(0) + \int_{0}^{1}(1-t)f''(t)dt\,,
 \end{equation}
 and the integral rest can be bounded using the previous inequality on $f''$ together with the triangle inequality,
 the elementary inequality $\cosh(a+b) \leq \cosh(a) \exp(b)$ for $b\geq 0$, and recalling $\norm{x} \leq c$:
 \begin{align*}
   \int_{0}^{1}(1-t)f''(t)dt & \leq \lambda^2 \norm{x}^2 B \int_0^1 (1-t) \cosh(\lambda (\norm{s}+t\norm{x})) dt\\
   &  \leq \lambda^2 \norm{x}^2 B \cosh(\lambda \norm{s}) \int_0^1 (1-t) \exp( \lambda tc) dt\\
   &  =  \norm{x}^2 B  \cosh(\lambda \norm{s}) \frac{\pfun(\lambda c)
   }{c^2}\,.
 \end{align*}
 Combining this with \eqref{eq:taylorf} and \eqref{eq:fprime} we get for $s\neq 0$:
 \begin{equation}
   \label{eq:detineq}
     \cosh(\lambda \norm{s+x}) =f(1) 
     \leq \cosh(\lambda \norm{s}) \Big( 1 + \lambda \inner{D_s \norm{.},x} 
    +  \norm{x}^2 B  \frac{
       \pfun(\lambda c)
     }{c^2}\Big)\,,
 \end{equation}
 where we have used $\sinh a \leq \cosh a$ in \eqref{eq:fprime}. The above inequality
 remains true for $s=0$ if we formally define $D_0 \norm{.}$ as $0$, due to $f'(0) = 0$ in this case, as
 argued earlier.

 We now go back to our initial goal of controlling $\e[1]{\cosh\paren[1]{\lambda\norm[1]{ \tilde{S}_k}}}$.
 We use the notation 
%
 $\ee[0]{j-1}{\cdot} := \e[0]{\cdot|\mathcal{M}_{i_{j-1}}}$
where $M_{i_{j-1}}= \sigma(X_{l}: 1 \leq l \leq i_{j-1}), l \in \mathbb{N}$ using $s:=\tilde{S}_{k-1}, x= X_{i_k}$ 
 then taking conditional expectations in \eqref{eq:detineq}, we obtain
 \begin{multline}
\label{eq:condexp}
 \ee[1]{k-1}{\cosh\paren[1]{\lambda\norm[1]{\tilde{S}_{k}}}}\\
 \leq \cosh\paren[1]{\lambda\norm[1]{\tilde{S}_{k-1}}} \paren{1 + \lambda \ee[1]{k-1}{\inner[1]{D_{\tilde{S}_{k-1}} \norm{.},
       X_{i_k}}} + \ee[1]{k-1}{\norm{X_{i_k}}^2} B \frac{\pfun(\lambda c)}{c^2} 
 }\,.
 \end{multline}
 

 In order to control the conditional expectation of the duality product on the right-hand side of \eqref{eq:condexp},
 we will need the additional following measure-theoretical lemma: 
  \begin{lemma}
   \label{lem:general_lemma}
   Assume $\mathcal{X},\mathcal{Y},\mathcal{T}$ are three Polish spaces.
   Let $F$ be a measurable real-valued function defined on  $\mathcal{X}\times\mathcal{T}$,
   and let $(X,Y)$ be a $\mathcal{X}\times\mathcal{Y}$-valued random variable
   ($\mathcal{X}\times\mathcal{Y}$ being endowed with its Borel sigma-algebra) on an underlying
probability space $(\Omega,\mathcal{F},\mathbb{P})$.
Denote through $B(t,\epsilon)$ an open ball of radius $\epsilon$, centered at point $t \in \mathcal{T}$. Assume 
that $F(X,t)$ is $\mathbb{P}$-integrable for all $t \in \mathcal{T}$ and that the following holds:
\begin{enumerate}
\item For all $t \in \mathcal{T}, \;\; \norm{\e{F(X,t)|Y} - \e{F(X,t)}}_{\infty} \leq C < \infty$\,; 
\item The mapping $t \mapsto F(x,t)$ is continuous in $t$ for all $x \in \mathcal{X}$; 
\item There exists $\eps>0$ and for all $t \in \mathcal{T}$ a measurable function $L_t(x):\mathcal{X} \rightarrow \mathbb{R}_+$ such that for all $x\in \mathcal{X}$,
$\sup_{t'\in B(t,\eps)}|F(x,t')| \leq L_{t}(x)$, and $L_{t}(X)$ is $\mathbb{P}$-integrable.
\end{enumerate}
Then, there is is a version of the conditional expectations $\e{F(X,t)|Y}$ such that
for $\mathbb{P}$-almost all $y \in \mathcal{Y}$, we have: 
\begin{equation}
  \label{eq:efxtcontrol}
\forall t \in \mathcal{T} \;\; \abs[1]{ \e{F(X,t)|Y=y} - \e[1]{ F(X,t) } } \leq C\,.
\end{equation}
In particular, if $\mathcal{T}=\mathcal{Y}$, under the previous assumptions we conclude
that 
\begin{equation}
  \label{eq:efxycontrol}
  \norm{ \e{F(X,Y)|Y} - \e[1]{F(\wt{X},Y)|Y}}_\infty \leq C\,,
\end{equation}
where $\wt{X}$ is a copy of $X$ which is independent of $Y$.
 \end{lemma}
(Observe that the whole point of this lemma is the inversion of quantificators ``for all $t$, for almost all $y$''
between its assumption (1) and the conclusion \eqref{eq:efxtcontrol}.)
 \begin{proof}[Proof of Lemma~\ref{lem:general_lemma} ]
Since $\mathcal{X}$ is Polish, 
there exists a regular conditional probability $\mathbb{P}(X \in \cdot|Y= \cdot)$, and
we choose as a particular version of all conditional expectations the pointwise integral with respect to this stochastic kernel.

By uniform local domination and continuity of $F$ in $t$ (Assumptions~2-3), the function $t\mapsto \e{F(X,t)}$
is continuous. Therefore, replacing $F$ by $\wt{F}(x,t) := F(x,t) - \e{F(X,t)}$ and $L_t$ by $2L_t$,
we can assume without loss of generality that $\e{F(X,t)}=0$ for all $t \in \mathcal{T}$. Since $\mathcal{T}$ is
assumed to be Polish, it is in particular separable; let $\wt{\mathcal{T}}$
be a countable dense subset of $\mathcal{T}$. From assumption (1), for each $\tilde{t} \in \wt{\mathcal{T}}$ there exists a measurable set $A_{\tilde{t}} \subset \mathcal{Y}$ with $\mathbb{P}\big(Y \in A_{\tilde{t}}\big) =1 $, such that $\abs[1]{\int_{\mathcal{X}}F(x,\tilde{t})d\mathbb{P}(x|Y=y)} \leq C$
for all $y \in A_{\tilde{t}}$. 
 Furthermore, for any $\tilde{t} \in \wt{\mathcal{T}}$, since the function $L_{\tilde{t}}(X)$ is
$\mathbb{P}$-integrable, it holds $\int_{\mathcal{X}}L_{\tilde{t}}(x) d\mathbb{P}(x|Y=y) < \infty$
  for all $y\in B_{\tilde{t}} \subset \mathcal{Y}$ with $\mathbb{P}\big(Y \in B_{\tilde{t}}\big) =1 $.

  This together with countability implies that the set $A:= \bigcap_{\tilde{t} \in \wt{\mathcal{T}}} (A_{\tilde{t}} \cap
  B_{\tilde{t}})$ is such that
  $\mathbb{P}\big(Y \in A\big) =1 $ and for all $(y,\tilde{t}) \in A \times \wt{\mathcal{T}}$, we have  $\abs[1]{\int_{\mathcal{X}}F(x,\tilde{t})d\mathbb{P}(x|Y=y) } \leq C$
  and $x \rightarrow L_{\tilde{t}}(x)$ is $\mathbb{P}(\cdot|Y=y)$-integrable.

For an arbitrary $t \in \mathcal{T}$, let  $\tilde{t}_{n}$ be a sequence of points in $\wt{\mathcal{T}}$ converging 
to $t$ in $\mathcal{T}$. We can assume without loss of generality that for all $n$, $d(\tilde{t}_n,t) < \eps/2$ (where
$\eps>0$ is from Assumption~3), so that $d(\tilde{t}_n,\tilde{t}_{n'}) \leq \eps$ for all $n,n'$, implying that
$\sup_n \abs{F(x,t_n)} \leq L_{\tilde{t}_1}(x)$ holds (by Assumption~3).
Now for all $y \in A$, using continuity (Assumption~2) we have for the conditional expectation under the regular conditional probability $\mathbb{P}(\cdot|Y=y)$, and 
by dominated convergence:
 \begin{align*}
\int_{\mathcal{X}}F(x,t)d\mathbb{P}(x|Y=y)  & = 
\int_{\mathcal{X}}\lim_{n \rightarrow \infty } F(x,\tilde{t}_{n})d\mathbb{P}(x|Y=y) \\
& = \lim_{n \rightarrow \infty } \int_{\mathcal{X}}F(x,\tilde{t}_{n})d\mathbb{P}(x|Y=y)
 \leq C\,.
 \end{align*}
 In the case $\mathcal{T}=\mathcal{Y}$, we note that  \eqref{eq:efxtcontrol} implies 
\eqref{eq:efxycontrol} by choosing $t=y$.
\end{proof}
Returning now to the proof of Lemma~\ref{lem:main_bound},
we use Lemma~\ref{lem:general_lemma} with $\mathcal{Y} = \mathcal{X}^{*}$,
 $F(x,y) = \langle y,x \rangle$, and $(X,Y) = (X_{i_k},D_{\tilde{S}_{k-1}}\norm{\cdot})$.
 By linearity of scalar product and expectation, and because the process $(X_i)_{i\geq 1}$ is centered,
 we have for fixed $y \in \mathcal{X}^{*}$:
 $\e{\langle y,X_{i_{k}}\rangle} = 0$. 
 Obviously $F$ is continuous in its first argument.
  Since by Assumption~\textbf{A2},
  $D_s\norm{\cdot}$ is uniformly bounded and $\mathcal{X}=B(c)$,
     we can restrict the domain of $F$ to $\mathcal{X} \times {B}^{\star}(A_1)$, and $F$ is then bounded uniformly, so that conditions (2) and (3) of Lemma~\ref{lem:general_lemma} are satisfied. Because of Assumption~\textbf{A3} it follows that we have $\norm{F(y,\cdot)}_{\mathcal{C}} \leq C_{1}\norm{F(y,\cdot)}_{\infty} \leq C_{1}A_{1}$. Finally, due to conditions on $\Phi_{\mathcal{C}}$-mixing coefficients, we have that condition~(1) is fulfilled with the constant $C=A_{1}C_{1}\Phi_{\mathcal{C}}(d_{k}) := \tilde{A}_{1}\Phi_{\mathcal{C}}(d_{k})$, so from \eqref{eq:efxycontrol} we conclude, that: 
  

 \begin{equation}
 \label{eq: f_null}
\abs{\ee[1]{k-1}{\inner[1]{D_{\tilde{S}_{k-1}} \norm{.},
       X_{i_k}}}}  \leq \tilde{A}_1 \Phi(d_{k})\,.
 \end{equation}
We turn to the control of the second conditional expectation on the right-hand side of \eqref{eq:condexp}.
 Using the $\Phi_{\mathcal{C}}$-mixing assumption and Assumptions~\textbf{A2,A3} again, we have almost surely (recalling
 $d_k:= i_k - i_{k-1}$):
 \begin{align*}
 \ee[1]{k-1}{\norm{X_{i_{k}}}^2} & \leq \ee[1]{k-1}{\norm{X_{i_{k}}}^2}- \ee[1]{}{\norm{X_{i_{k}}}^{2}} + \ee[1]{}{\norm{X_{i_{k}}}^{2}}\\
 & \leq C_{2} \Phi_{\mathcal{C}}(d_{k}) + \sigma^2, 
 \end{align*}
 since by Assumption~\textbf{A3}, the mapping $x \mapsto \norm{x}^2$ is bounded in semi-norm $C(f)$ on ${B}(c)$ by some constant $C_{2}$ 
 .
 Putting this bound together with \eqref{eq: f_null} in the inequality \eqref{eq:condexp}, 
we get: 
\[
\ee[1]{k-1}{\cosh\paren[1]{\lambda\norm[1]{\tilde{S}_{k}}}}
\leq \cosh\paren[1]{\lambda\norm[1]{\tilde{S}_{k-1}}} (1+p(d_k,\lambda))\,,
\]
where we recall $p(k,\lambda) := \lambda \tilde{A}_{1}\Phi_{\mathcal{C}}(k) + B \paren{ C_{2}\Phi_{\mathcal{C}}(k) + \sigma^2}\paren[2]{\frac{\pfun(\lambda c)}{c^2}}.$

 Iteratively repeating the aforementioned argument and considering that the bound on conditional expectation $\ee{k-1}{\cdot}$ holds almost surely, one obtains : 
 \begin{align*}
      \e[1]{\cosh\paren[1]{\lambda\norm[1]{\tilde{S}_{k}}}}
       &  = \e[1]{ \ee[1]{k-1}{\cosh \paren[1]{\lambda \norm[1]{\tilde{S}_{k}}}}} \\
      & \leq \e[1]{\cosh\paren[1]{\lambda\norm[1]{\tilde{S}_{k-1}}}} (1 + p(d_k,\lambda)) \notag \\
        & \leq  \e{\cosh\paren[1]{\lambda\norm[1]{X_{i_{1}}}}}\prod_{j=2}^{k}(1+p(d_j,\lambda)). \notag
 \end{align*} 
 For bounding $\e{\cosh\paren[1]{\lambda\norm[1]{X_{i_{1}}}}}$ we use \eqref{eq:detineq} with $s=0$ and obtain:
 \begin{align*}
    \ee{}{\cosh\norm{X_{i_{1}}}} &\leq \ee{}{1+\norm{X_{i_{1}}}^{2}B
      \frac{\pfun(\lambda c)}{c^{2}}} 
     \leq 1+ \frac{\sigma^{2}}{c^{2}} B \pfun(\lambda c)\,,
 \end{align*}
 which implies the claim.
\end{proof}
To proceed in the proof, we use the classical (\cite{Bosq:93},\cite{Wintenberger:10},\cite{Steinwart:16}) approach to divide the sample $(X_{1},\ldots,X_{n})$ into blocks, such that the distance between two neighbor elements in a given block will be large enough to ensure small dependence.
We partition the set $\{1,2,\ldots,n\}$ into $k$ blocks in the following way. Write $n=\ell k+r, 0 \leq r \leq k-1$ and define 
  \[
  	I_{i}=
  	\begin{cases}
  	
  	$\text{ $\{i,i+k,\ldots,i+\ell k \}$}$,&\text{if $ 1 \leq i\leq r $}\,,\\
  	$\text{ $\{i,i+k,\ldots,i+(\ell-1)k \}$}$ ,&\text{if $ r+1 \leq i\leq k $}\,.
  	\end{cases}
  \]  
  Denote through $|I_{i}|$ the number of elements in the $i-$th block; it holds $|I_{i}|=\ell+1$ for $1 \leq i \leq r$,
$|I_{i}|=\ell$ for $r+1 \leq i \leq k$, and  $\sum_{i=1}^{k}|I_{i}| = n$.  
Introduce the notation $S_{I_i} = \sum_{j \in I_i} X_{j}$.
     
   
     
 Now me may use Lemma~\ref{lem:main_bound}
 for each of the constructed blocks $I_i$, $1 \leq i \leq k$ to prove Lemma~\ref{lem:laplace_bound}.
 \begin{proof}[Proof of Lemma~\ref{lem:laplace_bound}]
   By the triangle inequality $\norm{S_{n}} \leq \sum_{j=1}^{k}\norm{S_{I_j}}$, implying
   for any $\lambda>0$, via the convexity of the exponential function:
  \begin{equation}
\label{eq: repr_i}
     \e[2]{\exp\paren[2]{\frac{\lambda}{n}{ \norm{S_{n}^{}}}}}
\leq    \e[3]{\exp\paren[3]{ \lambda \sum_{j=1}^{k} r_{j}\frac{\norm{S_{I_j}}}{|I_{j}|}}}
\leq \sum_{j=1}^{k}r_{j}\e[3]{\exp\paren[3]{ \frac{\lambda}{|I_{j}|}\norm[1]{S_{I_j}}}},
    \end{equation}
     where $r_{j} := \frac{|I_{j}|}{n}$, with $\sum_{j=1}^{k}r_{j} = 1$.
Now for each summand in the last sum, we apply Lemma~\ref{lem:main_bound} for the index tuple 
given by the ordered elements of $I_j$, yielding  
   \begin{equation*}
    \label{eq: part_i}
    \e[2]{ \exp\paren[2]{\frac{\lambda}{|I_{j}|} \norm[1]{S_{I_j}}}} 
    \leq 2 \paren{1+B\frac{\sigma^{2}}{c^{2}}
    \pfun\paren[2]{\frac{\lambda c}{|I_{j}|}}} \paren{1+p\paren{k,\tfrac{\lambda}{|I_j|}}}^{|I_{j}|-1}.
    \end{equation*}
Using this last bound into  \eqref{eq: repr_i}, we obtain: 
\begin{align*}
  \e[2]{\exp\paren[2]{\frac{\lambda}{n}{ \norm{S_{n}^{}}}}}
  & \leq 2\sum_{j=1}^{k}r_{j}\paren{1+B\frac{\sigma^{2}}{c^{2}}
\pfun\paren[2]{\frac{\lambda c}{|I_{j}|}}
  }\paren{1+p\paren{k,\tfrac{\lambda}{|I_j|}}}^{|I_{j}|-1}\\
  &   \leq  2 
  \sum_{j=1}^{k}r_{j}\exp\paren{\frac{B\sigma^{2}}{c^{2}}
    \pfun\paren[2]{\frac{\lambda c}{\ell}}
  }\exp\paren{\ell p\paren{k,\tfrac{\lambda}{\ell}}},
     \end{align*}
      where  we twice used the inequality $1+x \leq \exp(x)$, the condition $\ell \leq |I_{j}| \leq \ell+1$, and the fact that $p(k,\cdot)$ is non-decreasing in function for fixed $k$.
      The last quantity is equivalent to the claim of the lemma. 
\end{proof}

 \begin{proof}[Proof of Lemma~\ref{lem:prob_bound}]
   Using Chernoff's bound and Lemma~\ref{lem:laplace_bound}, we obtain for any $\lambda>0$:
  \begin{equation}
\label{eq:chernoffmain_2}
  \begin{aligned}
 \prob[2]{\tfrac{1}{n}{\norm{S_n}} \geq t }
 & = \prob[2]{\exp \paren{\tfrac{1}{n} {\norm[1]{\lambda S_{n}}}}  \geq \exp(\lambda t) } \\
 & \leq \exp(-\lambda t)\e[2]{\exp\paren{\tfrac{\lambda}{n}{ \norm{S_{n}^{}}}}}\\
 & \leq 2 \exp\paren{ -\lambda (t-\tilde{m}) + \tilde{\sigma}^2 \frac{(\ell+1)B}{c^2}
 \pfun\paren[2]{\frac{\lambda c}{\ell}}}\,,
  \end{aligned}
  \end{equation}
  where $\tilde{m} := \tilde{A}_1 \Phi_{\cC}(k)$ and $\tilde{\sigma}^2 := \sigma^2 + C_2 \Phi_{\cC}(k)$\,.

First we get an upper bound on the value of the function $\pi(\frac{\lambda c}{l})$. By using the Taylor series decomposition, simple inequality $2 \cdot 3^{k-2} \leq k!$ for $k \in \mathbb{N}$ and summing the geometric series we obtain: 
\begin{align*}
\pi\paren{\frac{\lambda c}{\ell}} &\leq \sum_{j=2}^{\infty} \paren{\frac{\lambda c}{\ell}}^{j}\frac{1}{2 \cdot 3^{j-2}} 
 = \frac{}{}\frac{\lambda^{2}c^{2}}{2\ell^{2}}\frac{1}{1- \frac{\lambda c}{3\ell}},
\end{align*}
where we assume that $0 < \lambda < \frac{3\ell}{c}$. Inserting this inequality into \eqref{eq:chernoffmain_2}\ and simplifying the terms we get:
\begin{align}
\label{eq:chernoffmain_3}
		\begin{aligned}
			\prob[2]{\tfrac{1}{n}{\norm{S_n}} \geq t } \leq 2 \exp\paren{ -\lambda (t-\tilde{m}) + \tilde{\sigma}^2  \lambda^{2}\frac{3(\ell+1)B}{2\ell}
				{}\frac{1}{3\ell- {\lambda c}{}}}.
		\end{aligned}
\end{align}

Now we put $\lambda = \frac{t\ell}{\frac{tc}{3} + \tilde{\sigma}^{2}B}$. Clearly, by this choice of $\lambda$ we have:
\begin{align*}
\frac{\lambda}{\ell} = \frac{t}{\frac{tc}{3}+\tilde{\sigma}^{2}B} \leq \frac{3}{c}.
\end{align*}  
Thus, the choice of $\lambda$ satisfies the assumption; putting it into the exponent of the right hand side of \eqref{eq:chernoffmain_3} we obtain:
\begin{multline*}
  -\lambda (t-\tilde{m})  + \frac{3}{2}\tilde{\sigma}^2  \frac{(\ell+1)B}{\ell}
  \lambda^{2}\frac{1}{3\ell- {\lambda c}{}} \\
  \begin{aligned}
    & = - \frac{t\ell(t-m)}{\frac{tc}{3}+\tilde{\sigma}^{2}B} + \frac{3}{2}\tilde{\sigma}^{2} \frac{\paren{\ell+1}B}{\ell}\frac{t^{2}\ell^{2}}{\paren{\frac{tc}{3}+\tilde{\sigma}^{2}B}^{2}} \frac{1}{3\ell - \frac{t \ell c}{\frac{tc}{3}+\tilde{\sigma}^{2}B}} \\
    & = - \frac{t\ell(t-\tilde{m})}{\frac{tc}{3}+\tilde{\sigma}^{2}B} + \frac{1}{2} \frac{\paren{\ell+1}t^{2}}{\frac{tc}{3}+\tilde{\sigma}^{2}B} \\
    & = - \frac{\paren{\ell-1}t^{2} - 2\ell\tilde{m}t}{2 \paren{\frac{tc}{3}+\tilde{\sigma}^{2}B}}.
  \end{aligned}
\end{multline*}
Putting this into the exponent bound and upper bounding $\ell$ with $2(\ell-1)$, for $\ell \geq 2$, we get the claim of the lemma.
\end{proof}

\begin{proof}[Proof of Theorem \ref{thm:general_result_all}]
  From the very last claim of Lemma \ref{lem:prob_bound} we have: 

\begin{equation*}
\probb[3]{}{\norm[2]{\frac{1}{n}\sum_{i=1}^{n}X_{i}} \geq t } \leq   2\exp\paren[3]{ - \frac{\ell\paren{t^{2} - 4\tilde{m}t}}{4 \paren{\frac{tc}{3}+\tilde{\sigma}^{2}B}}}.
\end{equation*}
Setting $\frac{\ell\paren{t^{2} - 4\tilde{m}t}}{4 \paren{\frac{tc}{3}+\tilde{\sigma}^{2}B}}:= \nu$ and solving the last equation in terms of $t$, we obtain: 

\begin{equation}
\label{eq: gen_result}
\probb[3]{}{\norm[2]{\frac{1}{n}\sum_{i=1}^{n}X_{i}}  \geq 4 \tilde{A_{1}}\Phi_{\mathcal{C}}(k) + 4\sqrt{\frac{B\tilde{\sigma}^{2}\nu}{\ell}}+ \frac{4}{3}\frac{c\nu}{\ell}}\leq 2\exp(-\nu),
\end{equation}
which proves the claim of the theorem.
\end{proof}

\begin{proof}[Proof of Theorem \ref{thm:main_result_gen}]
  From Theorem \ref{thm:general_result_all}, assuming the \textit{effective sample size} $\ell^{\star}$
  defined from \eqref{df:effective_sample_size} is greater than 2, and putting $C_* = C_2/C_1$, we obtain straightforwardly with probability at least $1-2\exp(-\nu)$:
	\begin{align}
			\begin{aligned}
				{\norm[3]{\frac{1}{n}\sum_{i=1}^{n}X_{i}}  \leq 4 {A_{1}}\paren{\frac{c}{\ell^{\star}}\vee \frac{\sigma}{\sqrt{\ell^{\star}}}} + 4\sqrt{\frac{B \paren{{\sigma}^{2} + C_* \paren{\frac{c}{\ell^{\star}}\vee \frac{\sigma}{\sqrt{\ell^{\star}}}}}\nu}{\ell^{\star}}}+ \frac{4}{3}\frac{c\nu}{\ell^{\star}}} =: \tilde{L}.	
			\end{aligned}
	\end{align}
		For $a,b > 0$ using the obvious inequalities $a\vee b \leq a +b$ and $\sqrt{ab}\leq (a+b)/2$ we obtain:
	\begin{align*}
	\tilde{L} & \leq {4A_{1} \paren{\frac{c}{\ell^{\star}} + \frac{\sigma}{\sqrt{\ell^{\star}}}} + 4 \frac{\sqrt{{B\nu}{}}\sigma}{\sqrt{\ell^{\star}}} + 4\frac{\sqrt{{BC_{*}c\nu}{}}}{\ell^{\star}} + 4\frac{\sqrt{{BC_{*}\sigma\nu}{}}}{\sqrt{\ell^{\star}\sqrt{\ell^{\star}}}} + \frac{4}{3}\frac{c\nu}{\ell^{\star}}} \\
	& \leq {4A_{1} \paren{\frac{c}{\ell^{\star}} + \frac{\sigma}{\sqrt{\ell^{\star}}}} + 4 \frac{\sqrt{{B\nu}{}}\sigma}{\sqrt{\ell^{\star}}} +2 \frac{\sqrt{{B\nu}{}}\paren{C_{*}+c}}{\ell^{\star}} +  2 \sqrt{{B\nu}{}}\paren{\frac{C_{*}}{\ell^{\star}}+ \frac{\sigma}{\sqrt{\ell^{\star}}}} + \frac{2c\nu}{\ell^{\star}}} \\
	&  \leq {\frac{\sigma}{\sqrt{\ell^{\star}}}\paren{4A_{1} + 6\sqrt{{B\nu}{}}} + \frac{c}{\ell^{\star}}\paren{2\nu + 2\sqrt{B\nu} + 4A_{1} + 4\sqrt{{B\nu}{}}\frac{C_{*}}{c}}}.
	\end{align*}
	Finally, observe that the inequality
	\begin{align*}
		\norm[2]{\frac{1}{n}\sum_{i=1}^{n}X_{i}} \leq {\frac{\sigma}{\sqrt{\ell^{\star}}}\paren{4A_{1} + 6\sqrt{{B\nu}{}}} + \frac{c}{\ell^{\star}}\paren{2\nu + 2\sqrt{B\nu} + 4A_{1} + 4\sqrt{{B\nu}{}}\frac{C_{*}}{c}}},
	\end{align*}
        trivially holds also for $\ell^{\star}=1$, since $A_1\geq 1$.
        This implies the statement of the theorem using $\sqrt{\nu}\leq \nu$, since we assumed $\nu\geq 1 $ here.
	
%

\end{proof}
Now we are equipped with all technical tools in order to prove the exponential bounds for different decay rates of the mixing coefficients.
\begin{proof}[Proof of Proposition \ref{thm:eff_sample_sizes}]
  We choose a reasonable bound $\ell_{g}$ on the effective sample size $\ell^{\star}$ in the case of geometrical mixing. Since $\Phi_{\cC}(\cdot)$ (extended to the positive real line as $\Phi_{\cC}(t) = \chi \exp(-(\theta t)^\gamma)$) is nonincreasing and $\frac{n}{2\ell} \leq \lfloor\frac{n}{\ell_{}}\rfloor$, 
  it is sufficient to choose $\ell_{g}$ such that $C_1 \Phi_{\mathcal{C}} \paren[1]{\frac{n}{2\ell_{g}}}$ is smaller than $\frac{c}{\ell_{g}} \vee \frac{\sigma}{\sqrt{\ell_{g}}}$. Moreover, in the case of geometrical mixing, it is sufficient to choose $\ell_{g}$ such that $C_1 \Phi_{\mathcal{C}} \paren[1]{\frac{n}{2\ell_{g}}} < \frac{c}{\ell_{g}}$ (trivially this implies that $C_1 \Phi_{\mathcal{C}} \paren[1]{\frac{n}{2\ell_{g}}} < \frac{c}{\ell_{g}} \vee \frac{\sigma}{\sqrt{\ell_{g}}} $). 
  We choose $\ell_{g} = \Big\lfloor \frac{n \theta}{2(1\vee\log\paren{{n \theta \chi C_1}/{c}})^{1/\gamma}}\Big\rfloor$. 
  It is easy to check that in this case, we get
  \begin{align*}
    {\ell_{g}}C_1 \Phi_{\mathcal{C}}\paren{\frac{n}{2\ell_{g}}} 
           \leq {n}\theta \chi C_1 \exp\paren{- 
          \paren{1 \vee \log \frac{\chi \theta C_1 {n}}{c}} }  \leq c,
	 \end{align*}
which together with the result of Theorem \ref{thm:main_result_gen} implies the first claim of the proposition.

For the case of polynomially mixing process, we have the coefficient decay rate $\Phi_{\mathcal{C}}(k) = \rho k^{-\gamma}$.
Similarly as above, we choose a bound $\ell_{p}$ for the \textit{effective sample size} $\ell^{\star}$ so that the conditions of Theorem \ref{thm:main_result_gen} are satisfied. 
Analogously, it is sufficient  to choose $\ell_{p}$ such that $C_1 \Phi_{\mathcal{C}}\paren[1]{\frac{n}{2\ell_{p}}} \leq \frac{\sigma}{\sqrt{\ell_{p}}} \vee \frac{c}{\ell_{p}}$.
Solving $C_1 \Phi_{\mathcal{C}}\paren{\frac{n}{2\ell}} \leq \frac{\sigma}{\sqrt{\ell_{}}} \vee \frac{c}{\ell}$ in $\ell$ for given $n$, $\sigma$, $c$, $\rho$, $C_1$  results in the following choice:
\begin{equation}
\label{eq:poly_size_ineq12}
		\ell_{p} =\max_{} \bigg\{\bigg \lfloor \paren{\frac{\sigma}{C_1 \rho}}^{\frac{2}{2\gamma +1}}\paren{\frac{n}{2}}^{\frac{2\gamma}{2\gamma +1}} \bigg \rfloor,\bigg \lfloor \paren{\frac{c}{C_1 \rho}}^{\frac{1}{\gamma +1}}\paren{\frac{n}{2}}^{\frac{\gamma}{\gamma +1}} \bigg \rfloor \bigg\},
\end{equation}
which matches the claim of the Proposition.
%
%
%
\end{proof}

\section{Justifications for Example~\ref{ex:shatten}}
\label{app:schatten}

We provide proof of the fact that the $p-$Schatten norm satisfies Assumption \textbf{A1} on the smoothness of the norm. 
\label{app:second}
        We recall that here $\cB$ is the space of real symmetric matrices of dimension $d$ and
        we use the notation $\norm{X}_{p}=\tr{\paren{\abs{X}^{p}}}^{\frac{1}{p}}$ for the Schatten $p-$norm of a symmetric matrix $X$.
We compute explicitly the upper bounds on $\abs[1]{\delta_{H}\paren[1]{\norm{X}_{p}}}$ and $\abs[1]{\delta_{H,H}\paren[1]{\norm{X}_{p}}}$.  

We make use of the following additional notation, which is standard functional
calculus over symmetric matrices. For a real diagonal matrix $W=\diag\paren{w_{1},\ldots,w_{d}}$ of dimension $d$, write $f\paren{W} = \diag\paren{f\paren{w_{1}}, \ldots, f\paren{w_{d}}}$,
where $f: I \subset{\mbr} \rightarrow \mbr$ is
a scalar function of class $C^{1}$ on $I$, and $I$ is a finite union of open intervals of $\mbr$,
containing the spectrum $\set{w_1,\ldots,w_d}$ of $W$.
For a symmetric matrix $X$ with the spectral decomposition $X= U\Lambda U^{\top} = \sum_{i=1}^{d}\lambda_{i}e_{i}e_{i}^{\top}$ we consider the matrix-valued maps $f(X) = U f\paren{\Lambda}U^{\top}$. Denote the real-valued function $g\paren{X} = \tr\paren{f(X)}$. Applying the chain rule and using Theorem V.3.3 from \cite{Bhatia:97}, we compute the Fr\'echet (and hence G\^{a}teaux) derivative of the function $g$ at point $X$ in the direction of an arbitrary matrix $H \in \mathcal{B}$. Namely, by linearity of the trace as a matrix operator, and from equations (V.9) and (V.12) from \cite{Bhatia:97} (which are stated there in the case where $I$ is an open interval,
but the extension to a finite union of open intervals is immediate), we deduce: 
\begin{align*}
  \delta_{H}\paren{g\paren{X}} &= \delta_{H}\paren{\tr{f\paren{X}}} = \frac{d}{dt} \bigg \vert_{t=0}\tr{f\paren{X+tH}} = \tr{\frac{d}{dt} \bigg \vert_{t=0}f\paren{X+tH} }\\
                               & =\tr{\delta_{H} f\paren{X}} = \tr\paren[1]{f^{[1]}\paren{\Lambda} \circ (U^\top H U)},
\end{align*}
where $\circ$ is used for the Hadamard (i.e. entry-wise) product of matrices;
and $f^{[1]}\paren{\Lambda}$ is a matrix whose $\paren{i,j}$ entry is defined as follows: 
\begin{align*}
  \paren{f^{[1]}\paren{\Lambda}}_{ij}= 
  \begin{cases}
    \frac{f\paren{\lambda_{i}} - f\paren{\lambda_{j}}}{\lambda_{i}-\lambda_{j}}, &\text{ if } \lambda_{i} \neq \lambda_{j} \\
    f^{'}\paren{\lambda_{i}} & \text{ otherwise.}
  \end{cases}
\end{align*}
Thus, denoting $\wt{H}=U^\top H U$,
we have: 
\begin{align*}
  \tr\paren[1]{f^{[1]}\paren{\Lambda} \circ \wt{H}}  
  &
    = \tr\paren[1]{f'(\Lambda)\circ \wt{H}}
    = \tr\paren[1]{f'(\Lambda) \wt{H}} = \tr\paren{f'(X) H} 
     = \inner[1]{f^{'}\paren{X},H}_{F},
\end{align*}
where the second to last equality follows from the definition of the matrix $f^{'}\paren{X}$ and $\inner{\cdot,\cdot}_{F}$ is the Frobenius product. This implies 
\begin{equation}
  \label{eq:der_trace}
  \delta_{H} \paren{\tr{f\paren{X}}} = \frac{d}{dt} \bigg \vert_{t=0} \tr{f\paren{X+tH}} = \inner[1]{f^{'}\paren{X},H}_{F},
\end{equation}  
so that the Fr\'echet-derivative of $\tr\paren{f\paren{X}}$ is $f^{'}\paren{X}$. (This formula is certainly not a novelty and its justification
included here for the sake of completeness.)

First consider the case where $X$ has full rank, therefore has no zero eigenvalue, and
apply Equation~\eqref{eq:der_trace} to the function $f: t \mapsto \abs{t}^{p}$
which is of class $\cC^1$ 
on $I=\mbr\setminus\set{0}$,
together with the chain rule to obtain that
\begin{align}
  \delta_{H}\paren[1]{\norm{X}_{p}} = \frac{d}{dt}\bigg \vert_{t=0} \norm{X + tH}_{p} = \frac{d}{dt}\bigg\vert_{t=0} \tr(f(X+tH))^{\frac{1}{p}} =  \inner[3]{\frac{w(X)}{\norm{X}_{p}^{p-1}}, H}_{F},
  \label{eq:firstgateaux}
\end{align}
where we introduced the notation $w(x)=\mathrm{sign}(x)\abs{x}^{p-1}$ on $I$.
From the definition of the Fr\'echet derivative, this implies that $D_{X}\paren[1]{\norm{\cdot}_{p}} := \frac{w(X)}{\norm{X}_{p}^{p-1}}$ is the corresponding Fr\'echet derivative at point $X$. Furthermore, for any $H \neq 0$ we have by the matricial Hölder's inequality: 
\begin{align*}
  \frac{\delta_{H}\paren[1]{\norm{X}_{p}}}{\norm{H}_{p}}  = \inner[3]{\frac{w(X)}{\norm{X}_{p}^{p-1}}, \frac{H}{\norm{H}_{p}}}_{F}  \leq 1,
\end{align*}
thus, for any $H \in \cB$ we have that $\abs[1]{\delta_{H}\paren[1]{\norm{X}_{p}} }\leq A_{1}\norm{H}_{p}$ with constant $A_{1}=1$.

For the second G\^{a}teaux differential, using linearity of the differential operator, we obtain: 
\begin{align*}
  \delta_{H,H} \paren[1]{\norm{\cdot}_{p}} = \delta_{H}\paren[1]{\delta_{H}\paren[1]{\norm{\cdot}_{p}}} = \delta_{H} \paren[1]{\inner[1]{D_{X}\paren[1]{\norm{\cdot}_{p}},H}} = \inner[1]{\delta_{H}\paren[1]{D_{X}\paren[1]{\norm{\cdot}_{p}}},H}.
\end{align*}
Furthermore, for $\delta_{H}\paren[1]{D_{X}\paren[1]{\norm{\cdot}_{p}}}$ we have by using the chain rule,  (V.9) and (V.12) from \cite{Bhatia:97} again, differentiation rules for matrices and Equation \eqref{eq:firstgateaux}: 
\begin{align*}
  \delta_{H}\paren[1]{D_{X}\paren[1]{\norm{\cdot}_{p}}}= \frac{d}{dt} \bigg\vert_{t=0} {D_{X+tH}\paren[1]{\norm{\cdot}_{p}}} & =  \frac{d}{dt} \bigg\vert_{t=0} \frac{w(X+tH)}{\norm{X+tH}^{p-1}_{p}} \\
                                                                                                                    & = \frac{U(w^{[1]}\paren{\Lambda}\circ \wt{H} ) U^\top}{\norm{X}_{p}^{p-1}} - \paren{p-1}\frac{w(X)\inner{w(X),H}}{\norm{X}_{p}^{2p-1}},
\end{align*}
where 
the matrix $w^{[1]}\paren{X}$ is defined analogously to $f^{[1]}\paren{X}$ before.
Therefore, for the second G\^{a}teaux differential we obtain explicitly: 
\begin{equation}
  \delta_{H,H}\paren[1]{\norm{X}_{p}} = \inner[1]{\delta_{H}\paren[1]{D_{X}\paren[1]{\norm{\cdot}_{p}}},H}= \frac{1}{\norm{X}_{p}^{p-1}}\inner[1]{w^{[1]}\paren{\Lambda}\circ \wt{H},\wt{H}}_{F} - \paren{p-1}\frac{\inner{w(X),H}^{2}}{\norm{X}_{p}^{2p-1}}.
  \label{eq:secondgateaux}
\end{equation}
For the second term, by the matricial H\"older's inequality we have: 
\begin{align*}
  \paren{p-1} \frac{\inner{w(X),H}_{F}^{2}}{\norm{X}_{p}^{2p-1}} \leq \paren{p-1} \frac{\norm{X}_{p}^{2p-2}\norm{H}^{2}_{p}}{\norm{X}_{p}^{2p-1}} = \frac{\norm{H}_{p}^{2}}{\norm{X}_{p}}. 
\end{align*}
For the first term, from the definition of the Hadamard product, we have
\[
  \inner{w^{[1]}\paren{\Lambda}\circ \wt{H}, \wt{H}}_{F} =  \sum_{i,j}^{}w^{[1]}\paren{\Lambda}_{ij}\wt{H}_{ij}^{2}.
\]
Furthermore, taking into account that $w'\paren{x} = \paren{p-1}\abs{x}^{p-2}$, by the mean value theorem on the closed interval $[\lambda_{j},\lambda_{i}]$ (assuming $\lambda_{i} > \lambda_{j}$), since the maximum of $w'$ is attained at one of the endpoints of the interval we have that
\[
  w^{[1]}\paren{\Lambda}_{ij} 
  \leq \frac{{w\paren{\lambda_{i}} - w\paren{\lambda_{j}}}}{ \lambda_{i} - \lambda_{j}} \leq \paren{p-1}\max\{\abs{\lambda_{i}}^{p-2}, \abs{\lambda_{j}}^{p-2}\}.
\]
In the case where $\lambda_{i}=\lambda_{j}$, by definition $w^{[1]}\paren{\Lambda}_{ij} = \paren{p-1}\abs{\lambda_{i}}^{p-2}$. Proceeding from this and using the symmetry of the matrix
$\wt{H}$ we get: 
\begin{align*}
\sum_{i,j}^{}w^{[1]}\paren{\Lambda}_{ij}\wt{H}_{ij}^{2}  
& \leq \sum_{i,j}^{} \paren{p-1}\max\{\abs{\lambda_{i}}^{p-2}, \abs{\lambda_{j}}^{p-2}\} \wt{H}_{ij}^2
  \\
& \leq \paren{p-1}\sum_{i,j}^{} \paren[1]{\abs{\lambda_{i}}^{p-2} + \abs{\lambda_{j}}^{p-2}} \wt{H}_{ij}^2 
  \\
& = 2 \paren{p-1} \sum_{i,j}^{}\abs{\lambda_{i}}^{p-2} \wt{H}_{ij}^2
  \\
  & =  2\paren{p-1}\sum_{i}^{}\abs{\lambda_{i}}^{p-2} \sum_j \wt{H}_{ij} \wt{H}_{ji}\\
  & =  2\paren{p-1}\sum_{i}^{}\abs{\lambda_{i}}^{p-2} \paren[1]{\wt{H}^2}_{ii}\\
& = 2\paren{p-1}\tr\paren[1]{\abs{\Lambda}^{p-2}\wt{H}^{2}}.
\end{align*}
Finally, applying the matricial H\"older's inequality once again for the last trace we get: 
\begin{align*}
\tr\paren{\abs{\Lambda}^{p-2}\wt{H}^{2}} = \inner{\abs{X}^{p-2},H^2}_F \leq \norm{X}_{p}^{p-2}\norm{H}_{p}^{2}.
\end{align*}
Gathering the above estimates, for the first term of \eqref{eq:secondgateaux} we obtain the following bound: 
\begin{align*}
  \frac{1}{\norm{X}_{p}^{p-1}} \inner{w^{[1]}\paren{X}\circ \wt{H},\wt{H}}_{F} \leq 2\frac{\paren{p-1}\norm{X}_{p}^{p-2} \norm{H}_{p}^{2}}{\norm{X}_{p}^{p-1}} = 2\paren{p-1} \frac{\norm{H}_{p}^{2}}{\norm{X}_{p}}.
\end{align*}
The latter implies that $\abs[1]{\delta_{H,H}\paren[1]{\norm{X}_{p}} } \leq A_2 \frac{\norm{H}_{p}^{2}}{\norm{X}_{p}}$ 
with $A_2=3(p-1)$ for all $H\in \cB$.

The inequalities required for Assumption \textbf{A1} are therefore established for all $X \in \cB$ of full rank.
To conclude the argument,
it was established in \cite{Pot:14} that $X \mapsto \norm{X}_p$ is of class $\cC_p$ for all
nonzero $X \in \cB$. Since full rank matrices are dense in $\cB$, by continuity
Assumption \textbf{A1} is satisfied for all nonzero $X \in \cB$.



\section{Proof of Lemma~\ref{lem:operator_deviation}}
\label{app:02}

As mentioned in the main body of the paper, in order to make use of the concentration inequalities for sums of random variables in ${\HS}(\mathcal{H}_{k})$ and in $\mathcal{H}_{k}$, 
we should ensure that the functions of interest of the original $\tau-$mixing process $Z_i=(X_i,Y_i)$
are again $\tau-$mixing. This claim is established by the Lipschitz property of the corresponding mappings in
the next lemma.

\begin{lemma}
  \label{lem:lipschitz_property}
  Assume $\cH_k$ is a RKHS over a base space $\cX$ (a closed ball of a Polish space) with reproducing kernel $k$ satisfying $\sup_{x \in \mathcal{X}}\sqrt{k(x,x)} \leq 1 $. Assume further the kernel admits a mixed partial derivative, $\partial_{1,2} k : \mathcal{X}\times \mathcal{X} \mapsto {\mathbb{R}}$ which is uniformly bounded
  on $\cX$ by an absolute constant $K>0$. Finally, let ${\cY}=[-R,R]$.
  Then, the mapping $V : \cX \rightarrow \HS(\cH_k) : x \mapsto  k_x \otimes k_x^*$ is $2K$-Lipschitz, and
  for a fixed $f\in \cH_k$, the mapping $W_f: \cX \times {\cY} \rightarrow \cH_k : (x,y) \mapsto y k_x - k_{x}\inner{k_x,f}$ is $3\max(KR,K\norm{f},1)$-Lipschitz. 
\end{lemma}
\begin{proof}[Proof of Lemma \ref{lem:lipschitz_property}]


As a starting point,  because of the assumption of uniform boundedness of the (mixed) partial derivative of the kernel $k$ and Lemma~3.3 from \cite{Scott:11}, we deduce that $k_{x}$ is $K-$Lipschitz as a map $\mathcal{X} \rightarrow \mathcal{H}_{k}$. Then, for arbitrary $x_{1},x_{2}$ we obtain: 
	%
	\begin{align}
	\begin{aligned}
	\label{eq:trace_lipschitz1}
		\norm{k_{x_{1}}\otimes k_{x_{1}}^{\star} - k_{x_{2}}\otimes k_{x_{2}}^{\star}}^{2}_{{\HS}_{}} &= \norm{k_{x_{1}}\otimes k_{x_{1}}^{\star} - k_{x_{1}}\otimes k_{x_{2}}^{\star} + k_{x_{1}}\otimes k_{x_{2}}^{\star}- k_{x_{2}}\otimes k_{x_{2}}^{\star} }^{2}_{{\HS}_{}} \\
			& = \norm{k_{x_{1}}\otimes \paren{k_{x_{1}}^{\star} - k_{x_{2}}^{\star}} }^{2}_{{\HS}_{}}+ \norm{\paren{k_{x_{1}}-k_{x_{2}}}\otimes k_{x_{2}}^{\star}}_{{\HS}}^{2} \\
			&   \;\;\;\;\;\;\;\;\;\;\; +2\inner{k_{x_{1}}\otimes\paren{k_{x_{1}}^{\star} - k_{x_{2}}^{\star}},\paren{k_{x_{1}}-k_{x_{2}}}\otimes k_{x_{2}^{\star}}} \\
			& \leq 
			\norm{k_{x_{1}}}_{\mathcal{H}_{k}}^{2}\norm{k_{x_{1}}^{\star}-k_{x_{2}}^{\star}}_{\mathcal{H}_{k}}^{2}+\norm{k_{x_{1}}-k_{x_{2}}}^{2}_{\mathcal{H}_{k}}\norm{k_{x_{2}}^{\star}}_{\mathcal{H}_{k}}^{2}\\
			&  \;\;\;\;\;\;\;\;\;\;\; +2\sqrt{\norm{k_{x_{1}}}_{\mathcal{H}_{k}}^{2}\norm{k_{x_{2}}}_{\mathcal{H}_{k}}^{2}}\norm{k_{x_{1}}- k_{x_{2}}}_{\mathcal{H}_{k}}^{2}	\\
			& \leq 4K^{2}
			\norm{x_{1}-x_{2}}^{2},
	\end{aligned}
	\end{align}
        where we used the properties of the Hilbert-Schmidt norm of tensor product operators, the Cauchy-Schwartz inequality in the third line, and the assumptions about boundedness of the kernel 
	$\norm{k_{x}}^{2}_{\mathcal{H}_{k}} = k(x,x) \leq 1 $ 
and that the map $k_{x}$ is $K-$Lipschitz in the last line.
%
%
Thus the map $V(x) = k_{x}\inner{k_{x},\cdot}$ is $2K$-Lipschitz.
Furthermore, we deduce
        \begin{equation}
          \label{eq:term2xi1}
          \norm{k_{x_1}\inner{k_{x_1},f}-k_{x_2}\inner{k_{x_2},f}}
          \leq \norm{k_{x_{1}} \otimes k_{x_{1}}^{\star} - k_{x_{2}} \otimes k_{x_{2}}^{\star} }_{{\HS}} \norm{f}
          \leq 2 K \norm{f}\norm{x_1-x_2}.
        \end{equation}      
	Quite analogously, for any $(x_{1},y_{1}), (x_{2},y_{2}) \in \mathcal{X} \times \mathcal{Y}$
        we have:
	\begin{align}
	\begin{aligned}
	\label{eq:adj_lipschitz}
	\norm{y_{1}k_{x_{1}}- y_{2}k_{x_{2}}}_{\mathcal{H}_{k}} & = \norm{y_{1}k_{x_{1}}- y_{1}k_{x_{2}} + y_{1}k_{x_{2}} - y_{2}k_{x_{2}}}_{\mathcal{H}_{k}}\\
	& = \norm{y_{1}(k_{x_{1}}-k_{x_{2}}) +k_{x_{2}}(y_{1} - y_{2}) }_{\mathcal{H}_{k}}\\
	& \leq KR\norm{x_{1}-x_{2}}_{\mathcal{X}}+|y_{1}-y_{2}| \\
	& \leq \max(KR,1)\paren{\norm{x_{1}-x_{2}}_{\mathcal{X}} + |y_{1}-y_{2}|}.
	\end{aligned}
	\end{align}
	The latter implies that the map $(x,y) \mapsto yk_{x}$ is $\max(KR,1)$-Lipschitz.
        By gathering bounds from \eqref{eq:term2xi1} and \eqref{eq:adj_lipschitz}, we deduce that
        $W_f(x,y):= yk_{x} - k_{x} \inner{k_{x}, f}$ is Lipschitz with constant $3\max(1,KR,K\norm{f})$  as a map $\mathcal{X}\times \mathcal{Y} \rightarrow
        \cH_k$.
      \end{proof}

\begin{proof}[Proof of Lemma~\ref{lem:operator_deviation}]
  Consider the mapping 
  \begin{align*}
	\xi_{1}(x,y) :=
	 {yk_{x} - k_{x} \inner{k_{x}, f_{\nu}}},
  \end{align*} with values in $\mathcal{H}_{k}$.
  It holds $\frac{1}{n}\sum_{i=1}^n \xi(x_i,y_i) = S_{\bx}^*y - T_\bx f_{\nu}$,
  as well as $\e{\xi_{1}(X,Y)}=0$.
  By Assumption~{\bf B1},
        \[
          \norm{\xi_{1}(x,y)} = \norm{yk_{x} - k_{x}\inner{k_{x},f_{\nu}}}_{\mathcal{H}_{k}} \leq  \norm{k_{x}}|y - f_{\nu}(x)|  
          \leq 2
          R.
\]
	
	
Due to Assumption~{\bf B1} and $\sup_{x \in \cX} k(x,x)\leq 1$, we obtain the following bound on the variance:
\begin{align*}
	\begin{aligned}
	\ee{}{\norm{\xi_{1}(X,Y)}^{2}} &=
	 \int_{\mathcal{X} \times \mathcal{Y}} \inner{k_{x}(y - f_{\nu}(x)), k_{x}(y - f_{\nu}(x))}d\nu(x,y)\\
	& =
	 \int_\mathcal{X}d\mu(x)k(x,x)\int_{Y}\paren{y-f_{\nu}(x)}^{2}d\nu(y|x) 
	\leq {
		\Sigma}^{2}.
	\end{aligned}
	\end{align*}
	Because of Lemma~\ref{lem:lipschitz_property}, and since $\xi_1(x,y) = W_{f_\nu}(x,y)$ with the notation
        there,
        if  $(X_{i},Y_{i})_{i\geq 1}$ is $\tau-$mixing  with rate $\tau(k)$, the sequence $\xi_{1}(x_{i},y_{i})_{i \geq 1}$is $\tau-$mixing with rate $\tau(k) = 3\max\paren{1,KR,KD}\tau(k)$. Using the result of Corollary~\ref{cor:tau_mix} with the  aforementioned bounds on the norm, the variance and the multiplicative correction for the mixing coefficients decay rate, we obtain with probability
        at least $1-\eta$:
	\begin{align*}
	\begin{aligned}
	\norm{{S_{k}}^{\star}y - {T}f_{\nu}} \leq 21\log(2\eta^{-1})
	\paren{{\frac{\Sigma^{}}{\sqrt{\ell_{1}}}}+ \frac{2R}{\ell_{1}}}, 
	\end{aligned}
	\end{align*}
	where the bound on the \textit{effective sample size} $\ell_{1}$ is obtained by straightforward plugging-in the bounds for the norm, the second moment and the form of mixing coefficients of the sequence $\xi_{1}(x_{i},y_{i})$ 
        in the general form given by Proposition~\ref{thm:eff_sample_sizes}, and is given by $\ell_{1} = \paren{\frac{
			\Sigma}{3\max(
			1,KR,KD))
			)\rho}}^{\frac{2}{2\gamma+1}}\paren{\frac{n}{2}}^{\frac{2\gamma}{2\gamma +1}}$,
		for a polynomially  mixing process with rate $\tau(k) = \rho k^{-\gamma}$, and $\ell_{1} = \bigg\lfloor \frac{n \theta}{2\paren{1 \vee \log\paren{n\frac{3\max(
					1,KR,KD
					)\chi\theta}{2
					R}}}^{\frac{1}{\gamma}}}\bigg\rfloor$
                                  for an exponentially mixing process with rate $\tau(k) = \chi\exp\paren{-(\theta k)^{\gamma}}$.

	The other inequalities will be derived in a similar way. 
	We introduce the random variable:
	\begin{align*}
	\xi_{2}(x,y) = \paren{{T} + \lambda}^{-\frac{1}{2}}
	\paren{k_{x}y-k_{x}\inner{k_{x},f_{\nu}}},
	\end{align*}
	Quite analogously, we can check that $\ee{}{\xi_{2}(X,Y)} = 0$.	
	Repeating similar steps, we get: 
	\begin{align*}
	\begin{aligned}
	\norm[1]{\paren{{T}+\lambda}^{-\frac{1}{2}}
		\paren{k_{x}y-k_{x}\inner{k_{x},f_{\nu}}}}_{\mathcal{H}_{k}}& \leq \norm[1]{\paren{{T}+\lambda}^{-\frac{1}{2}}}\norm{
		\paren{k_{x}y-k_{x}\inner{k_{x},f_{\nu}}}}_{\mathcal{H}_{k}} 
	\leq 2\lambda^{-\frac{1}{2}}
	R.
	\end{aligned}
	\end{align*}
	For the second moment of the norm of $\xi_{2}(X,Y)$, we get:
	\begin{align*}
	\begin{aligned}
	\ee{}{\norm{\xi_{2}(X,Y)}^{2}} &= \int_{\mathcal{X} \times \mathcal{Y}} \inner{\paren{{T} +\lambda}^{-\frac{1}{2}}k_{x}(y - f_{\nu}(x)), \paren{{T} +\lambda}^{-\frac{1}{2}}k_{x}(y - f_{\mathcal{H}_{k}}(x))}d\nu(x,y)\\
	& = \int_\mathcal{X}
	\norm{\paren{{T} +\lambda}^{-\frac{1}{2}}k_{x}}^{2}d\mu(x)\int_{\cY}\paren{y-f_{\nu}(x)}^{2}d\nu(y|x) \\
	& \leq \Sigma^{2}
	\int_{\mathcal{X}}Tr\paren{\paren{{T}+\lambda}^{-\frac{1}{2}}k_{x}\otimes k^{\star}_{x}}d\mu(x) \\
	& = \paren{
		\Sigma\sqrt{\mathcal{N}(\lambda)}}^{2}.
	\end{aligned}
	\end{align*}	
Using Lemma~\ref{lem:lipschitz_property}, one can readily see that the function $\xi_{2}(x,y)=(T+\lambda)^{-\frac{1}{2}}W_{f_\nu}(x,y)$ is Lipschitz with constant $3 \lambda^{-\frac{1}{2}}  \max(
		1,KR,KD)$,
                from which we deduce that
                $(\xi_{2}(X_{i},Y_{i}))_{i\geq 1}$
                is $\tau-$mixing with rate $3\lambda^{-\frac{1}{2}}\max\paren{
		1,KR,KD
	}
	\tau(k)$.
	Finally, by using Corollary \ref{cor:tau_mix}, we obtain with probability at least $1-\eta$: 
	\begin{align*}
	\begin{aligned}
	&\norm{\paren{{T}+\lambda}^{-\frac{1}{2}}\paren{{T}_{\bx}f_{\mathcal{H}_{k}} - {S}_{\bx}^{\star}y}}
 \leq 21\log\paren{\frac{2}{\eta}} \kappa^{-1}\paren{\frac{\Sigma\sqrt{\mathcal{N}(\lambda)}}{\sqrt{\ell_{2}}} + \frac{2R}{\sqrt{\lambda}\ell_{2}}}, 
	\end{aligned}
	\end{align*}
	where, as before, a bound on $\ell_{2}$ is obtained by Proposition \ref{thm:eff_sample_sizes} for either a polynomially or exponentially mixing process, through considering bounds on the norm, the second moment and the
        Lipschitz norm of the elements of the sequence $\xi_{2}(x_{i},y_{i})$.
	
	 
 We define the map $\xi_{3}: \mathcal{X} \mapsto \HS(\mathcal{H})$ (here, as mentioned before, through $\HS(\mathcal{H})$ we denote the space of Hilbert-Schmidt operators on $\mathcal{H}_{k}$) by:
	\begin{align*}
	\xi_{3}(x) := \paren{{T} + \lambda}^{-1}\paren{{T}_{x} -{T}},
	\end{align*}
	where we recall the notation ${T}_{x} = 
	k_{x} \otimes k_{x}^{\star}$ for any $x \in \mathcal{X}$. 
	
	Taking expectation we get: 
	\begin{align*}
	\begin{aligned}
	\ee{}{\xi_{3}(X)} = \paren{{T}+\lambda}^{-1}\int_{\mathcal{X}}\paren{{T}_{x}-{T}}d\mu(x) =0.
	\end{aligned}
	\end{align*}
	So that we have: 
	\begin{align*}
	\begin{aligned}
	\norm{\paren{{T}+\lambda}^{-1}\paren{{T} - {T}_{\bx}}} = \norm{\frac{1}{n}\sum_{i=1}^{n}\xi_{3}(x_{i})}.
	\end{aligned}
	\end{align*}
	Verifying the conditions as before we have: 
	\begin{align*}
	\begin{aligned}
	\norm{\xi_{3}(x_{})}_{\HS} &\leq \norm[1]{\paren{{T} + \lambda}^{-1}}\norm{{T}-{T}_{x}}_{\HS(\mathcal{H}_{k})}
	\leq 2\lambda^{-1},
	\end{aligned}
	\end{align*}
	and 
	\begin{align*}
	\begin{aligned}
	\ee{}{\norm{{\xi_{3}(X)}}_{\HS}^{2}} &= \int_{\mathcal{X}}Tr\paren{\paren{{T}_{x}-{T}}\paren{{T}+\lambda}^{-2}\paren{{T}_{x}-{T}}}\mu(dx) \\
	& = \int_{\mathcal{X}}Tr\paren{{T}_{x}\paren{{T}+\lambda}^{-2}{T}_{x}}\mu(dx) - Tr\paren{{T}_{}\paren{{T}+\lambda}^{-2}{T}_{}} \\
	& \leq \norm{{T} + \lambda}^{-1}\int_{\mathcal{X}}\norm{{T}_{x}}Tr\paren{\paren{{T} + \lambda}^{-1}{T}_{x}}\mu(dx) \\
	& \leq \lambda^{-1}\mathcal{N}(\lambda).
	\end{aligned}
	\end{align*}
	We can check via Lemma~\ref{lem:lipschitz_property} that $\xi_{3}$
        is Lipschitz with constant $2\lambda^{-1}K$,
	which  implies that $\paren{\xi_{3}(X_{i},Y_{i})}_{i\geq 1}$ is $\tau-$mixing with rate $2{\lambda 
}^{-1}K
\tau(k)$.
	
	We use the result of Theorem \ref{thm:main_result_gen}, applied to the quantity $\norm[1]{\frac{1}{n}\sum_{i=1}^{n}\xi_{3}(x_{i})}$. With probability at least $1-\eta$ we have:
	\begin{align*}
	\begin{aligned}
	& \norm{\paren{{T}+\lambda}^{-1}\paren{{T} - {T}_{\bx}}} \leq 	21\log\paren{\frac{2}{\eta}}  \paren{ \frac{\lambda^{-1/2}\sqrt{\mathcal{N}\paren{\lambda}}}{\sqrt{\ell_{3}}} + \frac{2\lambda^{-1}}{\ell_{3}}}.
	\end{aligned}
	\end{align*}
	where $\ell_{3}$ 
	is 	chosen following the standard "plug-in" scheme as before. 

	Finally, 	define $\xi_{4}(x) := 
	\paren{k_{x} \otimes k_{x}^{\star}-{T}}$. Again the random variables $\xi_{4}(X_{i})$ are centered and we have: 
	\begin{align*}
	\begin{aligned}
	{T}_{\bx} - {T} = \frac{1}{n}\sum_{i=1}^{n}\xi_{4}(x_{i}).
	\end{aligned}
	\end{align*}
	Repeating the scheme we get: 
	\begin{align*}
	\begin{aligned}
	\norm{\xi_{4}(x)}_{\HS(\mathcal{H})} & \leq 2,\\
	\ee{}{\norm{\xi_{4}(X)}_{\HS(\mathcal{H})}^{2}} &\leq 4,
	\end{aligned}
	\end{align*}
	
 Also, Lemma~\ref{lem:lipschitz_property} implies that $\xi_{4}(X_{i})_{i \geq 0}$ is $\tau-$mixing with rate $2K
 \tau(k)$,
	so that using the general deviation bound from Corollary~\ref{cor:tau_mix} according to the same principle as before, we obtain  with probability at least $1-\eta$:
	\begin{align*}
	\begin{aligned}
	\norm{{T}-{T}_{\bx}} &\leq  21\log\paren{\frac{2}{\eta}} \paren{\frac{2}{\sqrt{\ell_{4}}} + \frac{2}{\ell_{4}}}  
	 \leq \frac{42 \log\paren{{2}{\eta}^{-1}}}{\sqrt{\ell_{4}}}, 
	\end{aligned}
	\end{align*} 
where $\ell_{4}$ 
is chosen according to the mixing rate and bounds on the norm, variance term and Lipschitz constants as before.
\end{proof}

\section{Proofs of remaining results in Section~\ref{sec:application}}

We start with an auxiliary lemma, also in the same spirit as \cite{Bauer:09,Blanchard:17}.

\begin{lemma} Assume the conditions of Lemma~\ref{lem:operator_deviation} are satisfied.
	\label{lem:sup_norm_bound}
	Let $\eta \in (0,\frac{1}{2}]$ 
	 and $\lambda \in (0,1]$ be such that the following is satisfied:
	\begin{align*}
		\sqrt{\ell^{'}\lambda} \geq 50\log(2\eta^{-1})\sqrt{\max\paren{\mathcal{N}(\lambda),1}},
	\end{align*}
	with $\ell^{'}$ chosen to be the minimum of $\ell_{2},\ell_{3},\ell_{4}$ from Lemma~\ref{lem:operator_deviation}. 
 Then, with probability at least $1-\eta$, the following holds:
	\begin{align*}
			\begin{aligned}
					\norm{\paren{{T}_{\bx}+\lambda}^{-1}\paren{{T}+\lambda}} \leq 2.
			\end{aligned}
	\end{align*}
\end{lemma}

\begin{proof}[Proof of Lemma \ref{lem:sup_norm_bound}]
	By means of the Neumann series decomposition we write:
	\begin{align*}
	\begin{aligned}
	\paren{{T}_{\bx} + \lambda}^{-1}\paren{{T} +\lambda} = \paren{I- \Delta_\lambda}^{-1} =  \sum_{j=0}^{\infty}\Delta_\lambda^{j},
	\end{aligned}
	\end{align*}
	with $\Delta_\lambda := \paren{{T}+\lambda}^{-1}\paren{{T} - {T}_{\bx}} $. If $\norm{{T}_{\bx}(\lambda)} < 1$, then the last series converges and the norm of $\paren{{T}_{\bx} + \lambda}^{-1}\paren{{T} +\lambda}$ is bounded by the sum of the series of norms.
	From Lemma~\ref{lem:operator_deviation}, we have: 
	\begin{align*}
	\begin{aligned}
	\norm{\Delta_\lambda} \leq C_{\eta}\paren{\sqrt{\frac{\mathcal{N}(\lambda)}{\lambda \ell^{'}}} + \frac{2}{\lambda \ell^{'}}} ,
	\end{aligned}
	\end{align*}
	where we put $C_{\eta} =21\log(2\eta^{-1})$ for $\eta \in (0,\frac{1}{2}]$.
	Using the lemma's assumption and the fact that $C_{\eta} > 28$ for $\eta \in (0,\frac{1}{2}]$, we obtain: 
	\begin{align*}
	\begin{aligned}
	\sqrt{\lambda \ell^{'}} &\geq 2.3 C_{\eta}\sqrt{\max(\mathcal{N}(\lambda),1)} 
	\geq 2.3 C_{\eta} \geq 60.
	\end{aligned}
	\end{align*}
	This implies that 
	\begin{align*}
	\begin{aligned}
	\frac{1}{\lambda \ell^{'}} \leq \frac{1}{60\sqrt{\lambda \ell^{'}}}\leq \frac{1}{120 C_{\eta}}.
	\end{aligned}
	\end{align*}
	Putting these pieces together we obtain: 
	\begin{align*}
	\begin{aligned}
	\norm{\Delta_\lambda} &\leq C_{\eta}\paren{\frac{1}{2.3 C_{\eta}}+ \frac{1}{60 C_{\eta}}} < \frac{1}{2}.	
	\end{aligned}
	\end{align*}
	This implies, that with probability at least $1-\eta$:
	\begin{align*}
	\norm{\paren{{T}_{\bx} + \lambda}^{-1}\paren{{T} +\lambda}} \leq 2.
	\end{align*}
\end{proof}

\begin{proof}[Sketch of the proof of Lemma \ref{lem:error_bound_main}]
	The proof is analogous in form and spirit to that of Proposition~5.8 for the i.i.d. case given in \cite{Blanchard:17}. The main difference is reflected in using high probability upper bounds from Lemmata \ref{lem:operator_deviation} and \ref{lem:sup_norm_bound} instead of their i.i.d. counterparts, which in each case involve the knowledge of bounds on the effective sample size $\ell^{'}$. 
	The appropriate choice of the latter is assured by the two conditions from the theorem statement. Namely, $\ell^{'} \geq \ell_{0}$ implies the claim of Lemma~\ref{lem:sup_norm_bound} (which is the $\tau-$mixing counterpart of the Lemma 5.4 from \cite{Blanchard:17}). On the other hand, the condition $\ell^{'} \leq \min\{ \ell_{2},\ell_{3},\ell_{4}\}$ implies that all inequalities from Lemma~\ref{lem:operator_deviation} hold for $\ell^{'}$. We check additionally that the assumption $f_\nu \in \Omega({r,D})$ implies
        $\norm{f_{\nu}}\leq D$ (since $\norm{T}\leq 1$), which was a required condition for applying
        Lemma~\ref{lem:operator_deviation}.
        The remaining reasoning is the same as in Proposition (5.8) from \cite{Blanchard:17}.
	
\end{proof}
\begin{proof}[Proof of Theorem \ref{cor:balance_upper_bound} ]
The proof of the first part of the Theorem is in essense a direct extension of the proof of Corollary~5.9 in \cite{Blanchard:17} to the case of $\tau-$mixing stationary sequence.
	
As the marginal distribution $\mu$ belongs to the class $\mathcal{P}^{<}(b,\beta)$ (by assumption), from Proposition~3 in \cite{DeVito:06}, for any choice of parameter $\lambda \in (0,1]$ we obtain: 
\begin{equation}
  \label{eq:estdimeff}
  \mathcal{N}(\lambda) \leq \tilde{C}_{
  b,\beta}\lambda^{-\frac{1}{b}}.
\end{equation}
For the choice $\lambda_{n}$ and $\ell^{'}_{g}$ given by~\eqref{eq:cor_rate} as function of $n$
(the other parameters being fixed) it is easy to check by straightforward calculation that
$\ell_{g}^{'} \geq \ell_{0}$ holds, where $\ell_{0}$ is defined as in Lemma~\ref{lem:error_bound_main},
provided $n$ is larger than some $n_0$ (depending on all the fixed parameters).



Thus, as the given quantity $\ell^{'}_{g}$ fullfills all the requirements
of Lemma~\ref{lem:error_bound_main}, from this result we have with probability at least $1-\eta$:
\[ 
   \norm{T^{s}\paren{f_{\nu} - f_{\bz,\lambda_{n}}}}_{\mathcal{H}_{k}} 
   \leq \tilde{C}\log (8\eta^{-1})\lambda_{n}^{s}\paren[4]{D\paren[2]{\lambda_{n}^{r} + {\frac{1}{\sqrt{\ell^{'}_{g}}}}} + {\frac{R}{\ell^{'}_{g}\lambda_{n}} + \sqrt{\frac{\Sigma^{2}\lambda_{n}^{-\frac{b+1}{b} }}{\ell^{'}_{g}}}}  },
\]
where $\tilde{C} := C_{r,s,b,\beta,\overline{\gamma},E,B,\chi,\gamma}$ depends potentially on all model and method
parameters except for $R,D$ and $\Sigma$.

By direct computation, we check that the choice of regularization parameter sequence $\lambda_{n}$ implies that $\paren[1]{\ell_{g}^{'}}^{-1/2} = o_n\paren{\lambda^{r}_{n}}$. Therefore, for $n$ and therefore $\ell_{g}^{'}$ large enough, we can disregard the term $\paren[1]{\ell_{g}^{'}}^{-1/2}$ in the above bound,
if we multiply the front factor by 2.
In the same vein, we can check that $(\ell_{g}^{'}\lambda_{n})^{-1} = o_n \paren[1]{\paren[1]{\ell_{g}^{'}}^{-1/2}\lambda_{n}^{-\frac{b+1}{2b}}}$ and disregard the $R/(\ell'_g \lambda_n)$ term
for $n$ big enough.
Finally, the proposed choice of parameter $\lambda_{n}$ balances precisely the last two terms and leads to the conclusion.  
\end{proof}
\begin{proof}[Proof of Theorem \ref{cor:balance_upper_bound_pol}]
  In this proof $C_\bigtriangleup$ will denote a factor depending on the model and method parameters (but not
  on $n$ or $\eta$) whose exact value can change from line to line.
  
  Observe that estimate~\eqref{eq:estdimeff} still holds, and additionally due to the assumption of lower bounded spectrum, a matching lower bound for the effective dimension holds (with a different factor).
	Relagating the effects of the constants $R,K$ in the formulas from Table~\ref{tab:table1} in a generic factor, the choice of the bound for effective sample size $\ell^{'}_{p} = C_\bigtriangleup \paren{\lambda_{n} \mathcal{N}\paren{\lambda_n{}}}^{\frac{2}{2\gamma +1}} {n}^{\frac{2\gamma}{2\gamma +1}}$ ensures that condition $\ell^{'} \leq \min \{\ell_{2},\ell_{3},\ell_{4}\}$ is fullfilled with $\ell^{'} = \ell^{'}_{p}$ (which can be checked by straightforward computation) and $\lambda_{n}$ as defined by \eqref{eq:cor_rate_pol}. Furthermore, for $n>n_{0}$, where $n_{0}$ is as specified in the statement of the theorem, we obtain :
	\begin{align*}
	\log\eta^{-1} \leq C_{\bigtriangleup}n^{\frac{br}{2br+b+1+b(r+1)\gamma^{-1}}},
	\end{align*}
	which, by plugging in the value for $\lambda_{n}$ and estimate for $\mathcal{N}(\lambda_{n})$,
        implies that $\ell^{'} \geq \ell_{0}$, and we can apply Lemma~\ref{lem:error_bound_main}.

%
	 
	Thus, 
        we get with probability at least $1-\eta$:
	\begin{align*}
		& \norm{T^{s}\paren{f_{\nu} - f_{\bz,\lambda_{n}}}}_{\mathcal{H}_{k}} \\
	& \leq C_{\bigtriangleup,\eta}\lambda_{n}^{s}\paren{\paren{\lambda_{n}^{r} + \frac{\lambda_{n}^{-\paren{\frac{b-1}{2b}}\frac{1}{2\gamma +1}}}{n^{\frac{\gamma}{2\gamma +1}}}} + \frac{1}{\lambda_{n}^{1 + \frac{1}{2\gamma +1} \paren{\frac{b-1}{b}}}n^{\frac{2\gamma }{2\gamma +1}}} + \lambda_{n}^{-\frac{1}{b}\frac{\gamma(b+1)+b}{2\gamma+1}}n^{-\frac{\gamma}{2\gamma +1}}
	},
	\end{align*}
	where ${C}_{\bigtriangleup, \eta} = C_\bigtriangleup \log (8\eta^{-1})$.
	We observe that the choice of regularization parameter $\lambda_{n}$ implies that ${\lambda_{n}^{-\paren{\frac{b-1}{2b}}\frac{1}{2\gamma +1}}}/{n^{\frac{\gamma}{2\gamma +1}}} = o(\lambda^{r}_{n})$. Therefore, similarly to the case of exponentially $\tau-$mixing process, taking $n$ large enough and multiplying the front factor with $2$ we can disregard the term ${\lambda_{n}^{-\paren{\frac{b-1}{2b}}\frac{1}{2\gamma +1}}}/{n^{\frac{\gamma}{2\gamma +1}}}$ in the above bound. Similarly, one can check that: 
	\[
	\frac{1}{\lambda_{n}^{1 + \frac{1}{2\gamma +1} \paren{\frac{b-1}{b}}}n^{\frac{2\gamma }{2\gamma +1}}} = o\paren{\lambda_{n}^{-\frac{1}{b}\frac{\gamma(b+1)+b}{2\gamma+1}}n^{-\frac{\gamma}{2\gamma +1}}},
	\]
	so this term can be similarly asymptotically disregarded (again, multiplying the second term by $2$). Therefore, we can concentrate the analysis on the remaining main terms which are $\lambda_{n}^{r}$ and $\lambda_{n}^{-\frac{1}{b}\frac{\gamma(b+1)+b}{2\gamma+1}}n^{-\frac{\gamma}{2\gamma +1}}$. The choice of $\lambda_{n}$ balances exactly these terms and the computations lead to the conclusion.
%

\end{proof}

%
%
%
%
%
%

\bibliography{literature}
\end{document}